\pdfoutput=1

\documentclass[11pt, letterpaper]{article}

\usepackage[margin=1in]{geometry}
\usepackage{blindtext}

\usepackage{pifont}
\usepackage{hyperref}
\hypersetup{
	colorlinks=true,
	linkcolor=magenta!70!black,
	citecolor=blue!70!black,
	urlcolor=blue!70!black
}

\usepackage[english]{babel}
\usepackage[T1]{fontenc}

\usepackage{amsmath}
\usepackage{amssymb}
\usepackage{amsthm}
\usepackage{booktabs}
\usepackage{algorithm}
\usepackage[lined,boxed,ruled,norelsize,algo2e,linesnumbered]{algorithm2e}
\usepackage{bbm}
\usepackage{bbold}
\usepackage{parskip}
\usepackage{graphicx}
\usepackage{xcolor}
\graphicspath{{./figs/}}

\usepackage[nameinlink]{cleveref}
\usepackage{thm-restate}

\newtheorem{theorem}{Theorem}
\newtheorem{lemma}{Lemma}

\newtheorem{definition}{Definition}

\newtheorem{claim}{Claim}

\newtheorem{remark}{Remark}

\newcommand{\defeq}{:=}
\newcommand{\event}{\mathcal{E}}
\newcommand{\w}{\vec w}
\newcommand{\ws}{\w^\star}
\newcommand{\hw}{h_{\w}}
\newcommand{\hws}{h_{\ws}}

\newcommand\norm[1]{\left\| #1 \right\|}

\renewcommand\vec[1]{\mathbf{#1}}
\DeclareMathOperator*{\pr}{\mathbf{Pr}}
\DeclareMathOperator*{\E}{\mathbf{E}}
\renewcommand{\Pr}{\mathbf{Pr}}

\newcommand{\tO}{\widetilde{O}}

\def\multichoose#1#2{\ensuremath{\left(\kern-.3em\left(\genfrac{}{}{0pt}{}{#1}{#2}\right)\kern-.3em\right)}}

\newcommand{\cube}{\{\pm 1\}}

\newcommand{\B}{\mathbb{B}}
\newcommand{\T}{T_{\vw,\gamma}}
\newcommand{\Tt}{T_{\vw^t,\gamma}}

\newcommand{\g}{{\vec g}}

\newcommand{\R}{\mathbb{R}}

\newcommand{\N}{\mathbb{N}}
\newcommand{\eps}{\epsilon}

\newcommand{\poly}{\mathrm{poly}}

\newcommand{\sgn}{\mathrm{sign}}
\newcommand{\sign}{\mathrm{sign}}

\newcommand{\opt}{\mathrm{opt}}
\newcommand{\D}{D}

\newcommand{\Ind}{\mathbbm{1}}
\newcommand{\1}{\Ind}

\newcommand{\wt}{\widetilde}

\newcommand{\vg}{\vec g}

\newcommand{\x}{\vec x}
\newcommand{\vv}{\vec v}

\newcommand{\vw}{\vec w}
\newcommand{\z}{\vec z}

\newcommand{\citet}{\cite}
\newcommand{\citep}{\cite}

\newcommand{\grad}{\nabla}

\newcommand{\ball}{\mathbb{B}}
\newcommand{\ww}{\vec w}
\newcommand{\cmark}{\ding{51}}%
\newcommand{\xmark}{\ding{55}}%

\newcommand{\lzo}{\ell_{\textup{0-1}}}

\newcommand{\Par}[1]{\left(#1\right)}
\newcommand{\Brack}[1]{\left[#1\right]}
\newcommand{\Brace}[1]{\left\{#1\right\}}
\newcommand{\Abs}[1]{\left|#1\right|}

\newcommand{\supp}{\textup{supp}}
\definecolor{burntorange}{rgb}{0.8, 0.33, 0.0}

\newcommand{\Dx}{D_{\x}}

\newcommand{\Proj}{\boldsymbol{\Pi}}
\newcommand{\0}{\mathbb{0}}
\newcommand{\codeStyle}[1]{{\bfseries #1} }
\newcommand{\codeInput}{\codeStyle{Input:}}	
	
\newcommand{\codeReturn}{\codeStyle{Return:}}	
	
\newcommand{\half}{\frac{1}{2}}

\newcommand{\simu}{\sim_{\textup{unif.}}}
\newcommand{\lam}{\lambda}
\newcommand{\tOmega}{\widetilde{\Omega}}
\newcommand{\optrcn}{\opt_{\textup{RCN}}}
\newcommand{\Perspectron}{\mathsf{Perspectron}}

\title{Learning Noisy Halfspaces with a Margin: \\ Massart is No Harder than Random}

\author{Gautam Chandrasekaran\thanks{\texttt{gautamc@cs.utexas.edu}. Supported by the NSF AI Institute for Foundations of Machine Learning (IFML).} \\
	 UT Austin
   \and Vasilis Kontonis\thanks{\texttt{vasilis@cs.utexas.edu}. Supported by the NSF AI Institute for Foundations of Machine Learning (IFML).} \\
	 UT Austin
	 \and Konstantinos Stavropoulos\thanks{\texttt{kstavrop@cs.utexas.edu}. Supported by the NSF AI Institute for Foundations of Machine Learning (IFML) and by scholarships from Bodossaki Foundation and Leventis Foundation.} \\
	 UT Austin
  \and Kevin Tian\thanks{\texttt{kjtian@cs.utexas.edu}. Supported by the NSF AI Institute for Foundations of Machine Learning (IFML).} \\
	 UT Austin
   }
\date{}

\begin{document}

\maketitle
\begin{abstract}
 We study the problem of PAC learning $\gamma$-margin halfspaces with Massart noise. We propose a simple  proper learning algorithm, the Perspectron, that has sample complexity $\widetilde{O}((\epsilon\gamma)^{-2})$ and achieves classification error at most $\eta+\epsilon$ where $\eta$ is the Massart noise rate. 
 Prior works \cite{DiakonikolasGT19, ChenKMY20}  came with worse sample complexity
 guarantees (in both $\epsilon$ and $\gamma$) or could only
 handle random classification noise \cite{DiakonikolasDKWZ23, KontonisITBMV23} --- a much milder noise assumption. 
We also show that our results extend to the more challenging setting of learning generalized linear models with a known link function under Massart noise, achieving a similar sample complexity to the halfspace case. This significantly improves upon the prior state-of-the-art in this setting due to \cite{ChenKMY20}, who introduced this model.
\end{abstract}

\section{Introduction}
\label{sec:intro}

We study the problem of learning halfspaces with a margin, one of the oldest problems in the field of machine learning dating to work of Rosenblatt \cite{Rosenblatt58}. Specifically, we consider the following formulation of this problem, where the label distribution is corrupted by Massart noise \cite{Massart2006}, where we use the following notation for halfspace hypotheses $\hw: \R^d \to \cube$:
\begin{equation}\label{eq:halfspace_def}
\hw(\x) \defeq \sign\Par{\ww \cdot \x}, \text{ for } \w \in \R^d.
\end{equation}

\begin{definition}[Massart halfspace model]\label{defn:massart_half}
Let $\eta\in [0,\frac{1}{2}]$ and $\gamma \in (0, 1)$. We say that a distribution $D$ on $\B^d \times \cube$ is an instance of the \emph{$\eta$-Massart halfspace model with margin $\gamma$} if:
\begin{itemize}
\item There exists $\ws \in \R^d$ such that $\norm{\ws} = 1$,\footnote{This normalization is made for convenience, as the noise assumption of \Cref{defn:massart_half} is scale-invariant.} and $D_{\x}$ has margin $\gamma$ with respect to $\ws$,\footnote{See \Cref{sec:prelims} for our notation; $D_{\x}$ is the $\x$-marginal of $D$, and $D_{y}(\x)$ is the conditional marginal of $y \mid \x$.} i.e.,  
\(\pr_{\x\sim D_{\x}}[|\vw^\star\cdot \x|<\gamma]=0.\)
\item For all $\x \in \supp(D_{\x})$, there is an $\eta(\x) \in [0, \eta]$ such that
$
\pr\Brack{y \neq \hws(\x) \mid \x} = \eta(\x)
$.
\end{itemize}
\end{definition}

We note that \Cref{defn:massart_half} extends straightforwardly to general halfspaces up to rescaling (i.e., with larger domain size bounds and constant shift terms), as discussed in \Cref{rem:extensions}.

The \emph{Massart noise} model of \Cref{defn:massart_half} for halfspaces (and more generally, binary classification problems) has garnered interest from the statistics, machine learning, and algorithms communities for a variety of reasons. 
This noise model was originally introduced as an intermediate noise model, between the simpler (from an algorithmic design standpoint) \emph{random classification noise} (RCN) model \cite{AngluinL88}, and the more challenging \emph{agnostic} model \cite{Haussler92, KearnsSS94}.
In the RCN model, $\eta(\x)$ in \Cref{defn:massart_half} is restricted to be $\eta$ pointwise, i.e., the noise level is uniform; polynomial-time PAC learning has long since known to be tractable under RCN \cite{Bylander94, BlumFKV98}. On the other hand, in the agnostic model (where learning is computationally intractable under well-studied conjectures \cite{GuruswamiR06, FeldmanGKP06, Daniely16}), an adversary is allowed to arbitrarily modify an $\eta$ fraction of labels.
As observed by \cite{Sloan88}, the Massart noise model of \Cref{defn:massart_half} is equivalent to allowing an \emph{oblivious} adversary control an $\eta$ fraction of labels, where the $\eta$ fraction is crucially sampled independently at random. It was stated as a longstanding open question \cite{Cohen97, Blum03} whether this obliviousness of the adversary impacts the polynomial-time tractability of learning halfspaces, even with a margin.

For additional motivation, it is reasonable to consider Massart noise to be a more realistic model of real-life noise (even when benign) when compared to the RCN model, as it allows for some amount of non-uniformity. This made \Cref{defn:massart_half} a possibly tractable way to relax the noise assumption, without running into the aforementioned computational barriers for agnostic learning. In a series of recent exciting developments, in large part spurred by the breakthrough work of \cite{DiakonikolasGT19} who gave an (improper) polynomial-time PAC learning algorithm in the Massart halfspace model, significant algorithmic advances have been made towards understanding the polynomial-time tractability of learning under Massart noise 
\cite{AwasthiBHU15,DiakonikolasGT19,DKTZ20,ChenKMY20, zhang2021improved, DKKTZGeneral}. However, less is understood about the fine-grained sample and computational complexity of these problems, which is potentially of greater interest from a practical perspective.

We investigate this question of fine-grained complexity for the Massart halfspace model, inspired by a line of recent work on \emph{semi-random models} \cite{BlumS95}, a popular framework for understanding the overfitting of algorithms to their modeling assumptions. To motivate semi-random models, observe that from a purely information-theoretic standpoint, one might suspect that learning under Massart noise is actually \emph{easier} than RCN; the noise level $\eta(\x)$ is only allowed to decrease, giving more ``signal'' with respect to $\ws$. However, this modification poses challenges when designing algorithms, e.g., because it breaks independence between $y$ and $\x$ beyond the value $\sign(\ws \cdot \x)$. Indeed, this is reflected in our current knowledge of halfspace learning algorithms. While it is known that one can learn halfspaces with margin $\gamma$ under the RCN model to $\eps$ error (in the zero-one loss) using $\tO((\eps\gamma)^{-2})$ samples \cite{DiakonikolasDKWZ23, KontonisITBMV23}, state-of-the-art learners under Massart noise use $\tO(\gamma^{-4}\eps^{-3})$ samples (if required to be proper) \cite{ChenKMY20} or $\tO(\min(\gamma^{-4}\eps^{-3}, \gamma^{-3}\eps^{-5}))$ samples (otherwise) \cite{DiakonikolasGT19}. The semi-random model framework posits that this reflects a lack of robustness in current sample-efficient algorithms for learning halfspaces, due to their overfitting to the RCN assumption.

For many statistical learning problems, new algorithms have been developed under semi-random modeling assumptions, with guarantees matching, or nearly-matching, classical algorithms under the corresponding fully random models \cite{ChengG18, KelnerLLST23, JambulapatiLMSST23, GaoC23, BlumGLMSY24}. This leads us to our motivating problem, which aims to accomplish this goal for learning halfspaces.
\begin{gather*}
\textit{Can we design algorithms for learning in the Massart halfspace model with} \\
\textit{sample complexities matching the state-of-the-art for learning in the RCN model?}
\end{gather*}

\subsection{Our results}
\label{ssec:results}

As our main contribution, we resolve this problem in the affirmative, in the setting of~\Cref{defn:massart_half}. We also extend our results to a substantial generalization of this model in~\Cref{defn:massart_glm}. 

\paragraph{Massart halfspace model.} 
We begin with our basic result in the setting of the Massart halfspace model,~\Cref{defn:massart_half}. Our goal in this setting is to find a \emph{proper} hypothesis halfspace $h_{\w}(\x) = \sign(\w \cdot \x)$ for $\w \in \B^d$, achieving good zero-one loss $\lzo$ (see Section~\ref{sec:prelims} for a definition) over examples $(\x, y)$ drawn from the distribution $D$. Our main result to this end is the following.
\begin{theorem}[Informal, see~\Cref{theorem:halfspace-massart}]
    \label{thm:halfspace_intro}
    Let $D$ be an instance of the $\eta$-Massart halfspace model with margin $\gamma$, and let $\eps \in (0, 1)$. Then,  $\mathsf{Perspectron}$ (\Cref{alg:massart_margin}) returns $\w \in \B^d$ such that $\lzo(\w) \le \eta + \eps$ with probability $0.99$,\footnote{The formal variant, \Cref{theorem:halfspace-massart}, gives high-probability bounds at a mild polylogarithmic overhead in sample and runtime complexities.} using $\tO(\gamma^{-2}\eps^{-2})$ samples  and $\tO(d\gamma^{-2}\eps^{-4})$ time.
\end{theorem}

We pause to comment on~\Cref{thm:halfspace_intro}. First, our error guarantee is of the form $\eta + \eps$ rather than the more stringent goal of $\lzo(\ws) + \eps = \E_{\x \sim \Dx}[\eta(\x)] + \eps$. There is strong evidence that this distinction is necessary for polynomial-time algorithms in the statistical query (SQ) framework of \cite{Kearns:98}, which our algorithm is an instance of, due to \cite{ChenKMY20,DK20-SQ-Massart,NasserT22}. Next, the sample complexity bound of~\Cref{thm:halfspace_intro} matches the results of \cite{DiakonikolasDKWZ23,KontonisITBMV23}, the state-of-the-art under the milder RCN model. There is evidence that the dependences in~\Cref{thm:halfspace_intro} on both $\eps^{-1}$ and $\gamma^{-1}$ are individually tight. In particular, \cite{Massart2006} shows the sample complexity of the problem is $\tOmega(\gamma^{-2}\eps^{-1})$, and \cite{DiakonikolasDKWZ23} shows any efficient algorithm in the SQ framework must use $\tOmega(\gamma^{-1/2}\eps^{-2})$ samples. We also remark that we can assume without loss of generality that $\eta$ is known (see \Cref{sec:tolerant_noise}).

Finally, as mentioned, prior to our work, the best-known polynomial-time learners under \Cref{defn:massart_half} had sample complexities $\tO(\min(\gamma^{-4}\eps^{-3}, \gamma^{-3}\eps^{-5}))$ \cite{ChenKMY20,DiakonikolasGT19}. In Table~\ref{tab:results}, we summarize relevant sample complexity bounds for learning variants of halfspace models with noise.

\renewcommand{\arraystretch}{1.5}
\begin{table}[ht]
\centering
\begin{tabular}{|c|c|c|c|c|}
\hline
Source & RCN & Massart & Proper& Sample Complexity\\
\hline
\hline
 \cite{DiakonikolasGT19} & \cmark & \xmark & \cmark & $\gamma^{-4}\eps^{-2}$\\
\hline
\cite{DiakonikolasDKWZ23,KontonisITBMV23}& \cmark & \xmark & \cmark& $\gamma^{-2}\eps^{-2}$\\
\hline
\cite{DiakonikolasGT19}& \cmark & \cmark & \xmark& $\gamma^{-3}\eps^{-5}$\\
\hline
\cite{ChenKMY20} &\cmark & \cmark & \cmark& $\gamma^{-4}\eps^{-3}$\\
\hline
\textbf{\Cref{thm:halfspace_intro}}& \cmark & \cmark & \cmark& $\gamma^{-2}\eps^{-2}$\\
\hline
\end{tabular}
\caption{Sample complexities of learning halfspaces with $\gamma$ margin, omitting logarithmic factors and failure probabilities for brevity. All the algorithms above run in polynomial time.}
\label{tab:results}
\end{table}

%

\paragraph{Massart generalized linear models.} 

Our second result is an extension of~\Cref{thm:halfspace_intro} to the more challenging setting of learning generalized linear models (GLMs) with a known link function $\sigma$ under Massart noise. As before, we only consider distributions that have a margin with respect to the optimal halfspace. We now state a slightly-simplified variant of our setting (cf.\ \Cref{rem:sig_odd}).

\begin{definition}[Massart GLM, simplified]
\label{defn:massart_glm}
Let $\sigma:[-1, 1]\to [-1,1]$ be an odd, non-decreasing function. We say that a distribution $D$ on $\B^d \times \cube$ is an instance of the \emph{$\sigma$-Massart generalized linear model (GLM) with margin $\gamma$} if the following conditions hold.
\begin{enumerate}
  \setlength{\itemsep}{2pt}
  \setlength{\parskip}{0pt}
    \item There exists $\ws \in \R^d$ such that $\norm{\ws} = 1$ and $\pr[|\vw^\star\cdot \x|< \gamma]=0$.
    \item For all $\x \in \supp(\Dx)$, it holds that
    $\eta(\x) \defeq \pr[y \neq \hws(\x) \mid \x]
    \leq \frac{1 - |\sigma(\ws \cdot \x)|}{2}$.
\end{enumerate}
\end{definition}
\begin{remark}\label{rem:sig_odd}
The assumption that $\sigma$ is odd is also commonly used in prior works (see, e.g., \cite{CN08, DKTZ20, ChenKMY20}). We further show that our result extends to $\sigma$ with bounded asymmetry, albeit with a weaker error guarantee (see \Cref{defn:massart_glm_appendix} and \Cref{theorem:glm-massart}).
\end{remark}
To provide intuition for \Cref{defn:massart_glm}, observe that that if $\eta(\x) = \frac{1 - |\sigma(\ws \cdot \x)|}{2}$ for some $\x \in \supp(\Dx)$, then
$\E\Brack{y\mid \x}
= |\sigma(\ws \cdot \x)|\sign(\ws \cdot \x) = \sigma(\ws \cdot \x)$, i.e., \Cref{defn:massart_glm}
corresponds to the standard GLM definition.
In \Cref{defn:massart_glm} (compared to \Cref{defn:massart_half}), we replace the fixed noise rate upper bound $\eta$ with a data-dependent upper bound which is monotone (i.e., decreases as $|\ws \cdot \x|$ grows more confident). 
That is, \Cref{defn:massart_glm} 
generalizes the problem of learning Massart halfspaces, which follows
by taking $\sigma(t) = (1 - 2\eta) \sign(t)$ for all $t \in [-1, 1]$.

When working with a Massart GLM $D$, we define
\[\optrcn \defeq \E_{(\x, y) \sim D}\Brack{\frac{1 - |\sigma(\ws \cdot \x)|}{2}}.\]
Note that in the special case of a Massart halfspace model, we simply have $\optrcn = \eta$. As in the case of Massart halfspaces, known SQ lower bounds make competing with $\opt \defeq \lzo(\ws) \defeq 
 \pr_{(\x, y) \sim D}\Brack{y \neq \hws(x)}$
 an intractable target, so our focus is again on attaining $\lzo(\w) \approx \optrcn$. 

To our knowledge, the model in \Cref{defn:massart_glm} was first studied in \cite{ChenKMY20}, though we note that similar models have been considered in prior works \cite{ZhangLC17, HKLM20, DKTZ20}, which we describe and compare to \Cref{defn:massart_glm} in Section~\ref{sec:related_work}.
In \cite{ChenKMY20}, the parameterization of this model is slightly different; they assume $\sigma$ is $L$-Lipschitz and that $\Pr_{\x \sim \Dx}[|\sigma(\vw^\star \cdot \x)| \ge \gamma] = 1$, i.e., they impose a margin on $\sigma(\vw^\star \cdot \x)$ rather than $\vw^\star \cdot \x$. This implies our margin assumption (with $\gamma \gets \frac \gamma L$ in \Cref{defn:massart_glm}), but not vice versa, so our \Cref{defn:massart_glm} is slightly more general.

Under their slightly more restrictive assumptions, \cite{ChenKMY20} claims a runtime which is an unspecified polynomial in $L\gamma^{-1}\eps^{-1}$, that is at least $\tOmega(L^4 \gamma^{-4}\eps^{-6})$ when specialized to the halfspace case (see their Theorems~{5.2} and {6.14}). On the other hand, we achieve improved rates extending our simple algorithmic approach for the Massart halfspace case in this more challenging setting.

\begin{theorem}[Informal, see \Cref{theorem:glm-massart}]
    \label{theorem:glm_intro}
Let $D$ be an instance of the $\sigma$-Massart GLM with margin $\gamma$, and let $\eps \in (0, 1)$. There is an algorithm returning $\w \in \B^d$ so that $\lzo(\w) \le \optrcn + \eps$ with probability $0.99$, using $\tO(\gamma^{-2}\eps^{-4})$ samples and $\tO(d\gamma^{-2}\eps^{-6})$ time.
\end{theorem}
In particular, parameterizing our problem using the margin and Lipschitz assumptions in \cite{ChenKMY20} (with $\gamma \gets \frac \gamma L$), we obtain an improved sample complexity of $\tO(L^2\gamma^{-2}\eps^{-4})$.

\subsection{Technical overview}
\label{ssec:tech_overview}

\paragraph{Learning Massart halfspaces.}
\label{sec:intro_massart_overview}
Our learning algorithms are inspired by the certificate framework
for learning with semi-random noise developed in \cite{DKTZ20,ChenKMY20}.  In that framework, given a sub-optimal hypothesis $\vec w$, i.e.,
with error $\pr_{(\x, y) \sim D}[\sgn(\vec w \cdot \x) \neq y]
\geq \eta + \eps$,
the goal is to construct a certificate of sub-optimality in the form of a separating hyperplane between $\vec w$ and the target $\vw^\star$, i.e., a vector $\vg$ such that $\vg \cdot \w \ge \vg \cdot \ws \iff \vg \cdot (\w - \ws) \ge 0$.  Given such a separating hyperplane, prior works
rely on cutting-plane methods (e.g., \cite{Vaidya:96})
or on first-order regret minimization methods to learn a hypothesis achieving the target error.

We first describe our algorithm in the halfspace setting, by motivating our choice of a certificate. Prior work \cite{ChenKMY20} achieving a proper Massart halfspace learner uses the gradient of the Leaky-ReLU 
objective
$\ell_{\eta}(t)\defeq
(1-\eta) \max(0, t)
- \eta \max(0, -t)$
conditioned on a band around the current hypothesis $\vec w$ as a separating hyperplane. That is, they argue that for some appropriate interval $I$, it holds that
$\E[ \nabla \ell_{\eta}(- y \vec w \cdot \x ) \mid \x \cdot \vec w \in I]
\cdot (\vec w-\vw^\star) \geq 0$.  This yields a sample complexity scaling as $\tilde{O}(\epsilon^{-3})$ because one has to sample condionally from the band $I$ and estimate the gradient of the Leaky-ReLU on these samples up to error $\epsilon$, as well as additional overhead in $\gamma^{-1}$ due to use of expensive outer loops taking advantage of these certificates, such as cutting-plane methods.

To avoid the sample complexity overhead of this conditioning (implemented via rejection sampling), we use a simple reweighting scheme pointwise, which puts a larger weight of $|\vw \cdot \x|^{-1}$ (an inverse margin) on points closer to the boundary of our current hypothesis $h_{\w}$. Intuitively, this reweighting is a soft implementation of the hard conditioning done in \cite{ChenKMY20}. This is motivated by our first important observation: any significantly sub-optimal hypothesis $\w$ with $\lzo(\w) \ge \eta + \eps$ satisfies
\[\E\Brack{\frac{\nabla \ell_\eta(-y\w \cdot \x)}{|\w \cdot \x|}} \cdot \Par{\w - \ws}  \geq \eps,\]
proven in \Cref{claim:warmup_separating_hyperplane}. This suggests using $\vg = \E[\nabla \ell_\eta(-y\w \cdot \x) |\w \cdot \x|^{-1}]$ (rather than $\E[\nabla \ell_\eta(-y \w \cdot \x) \mid \x \cdot \w \in I]$ as in \cite{ChenKMY20}) as our certificate, which we can estimate via a single sample.

While reweighting by the inverse margin $|\vw \cdot \x|^{-1}$ gives a separating hyperplane certificate, it may be impossible to estimate this certificate from few samples, e.g., if the weight $|\vw \cdot \x|^{-1}$ is often very large, which introduces significant variance. 
For instance, even if $\Dx$ has margin with respect to a target $\ws$,
this may not hold with respect to the current hypothesis $\vec w$ (without further distributional assumptions on $\D_\x$), so $|\vw \cdot \x|$ can in principle be arbitrarily small.  To overcome this, we 
change our pointwise reweighting to be less 
aggressive and instead use $(|\vec w \cdot \x| + \gamma)^{-1}$.
In \Cref{lem:gradient_lb}, we exploit the margin assumption about $\Dx$ to show that when $\lzo(\w) \ge \eta + \eps$, it is still the case that
\[\E\Brack{\frac{ \nabla \ell_{\eta}(- y \vec w \cdot \x ) }{|\vec w \cdot \x| +\gamma} }\cdot( \vec w - \vw^\star)\geq \eps.\] Moreover, we still have an unbiased estimator for $\E[ \nabla \ell_{\eta}(- y \vec w \cdot \x )  (|\vec w \cdot \x| +\gamma)^{-1} ]$, and the estimator is bounded in Euclidean norm by $\gamma^{-1}$ with probability $1$ by our margin assumption.

Standard concentration inequalities now show $\tO(\gamma^{-2} \eps^{-2})$ samples suffice to obtain a separation oracle with high probability. 
Plugging this certificate into oracle-efficient cutting-plane methods (e.g., \cite{Vaidya:96}) implies an algorithm with sample complexity $\wt{O}(\gamma^{-4} \eps^{-2})$ (after a random projection process \cite{ArriagaVempala:99} to reduce to $\wt{O}(\gamma^{-2})$ dimensions).  This already improves
upon the prior best-known $\wt{O}(\gamma^{-4} \eps^{-3})$ sample complexity.  We further improve
our dependence on $\gamma$ by using it in a perceptron-like regret minimization scheme, where
at every step we update the current hypothesis $\vec w$ using the aforementioned bounded unbiased estimator of our certificate, see \Cref{lem:sgd}
and \Cref{alg:massart_margin}.  Overall, our algorithm iterates the following \emph{very simple update} for a step size $\lambda > 0$
and $\beta \defeq 1-2\eta$:
\begin{equation}\label{eq:perspectron}
\vec w^{(t+1)} \gets \vec w^{(t)} - \lambda (\beta \sgn(\vec w^{(t)} \cdot \x^{(t)}) - y^{(t)}) \frac{\x^{(t)}}{|\vec w^{(t)} \cdot \x^{(t)}| + \gamma}
~~~
\text{ with } 
~~~
\vec w^{(0)} \gets \vec 0 \,.
\end{equation}
Since at every step we perform an (approximate) perspective projection $(|\vec w \cdot \x| + \gamma)^{-1}$ on our sample, we call Algorithm~\ref{alg:massart_margin} which iterates the update in \Cref{eq:perspectron} the $\Perspectron$. 

\paragraph{Learning Massart GLMs.}
\label{sec:intro_massart_glm}
For learning Massart GLMs (\Cref{defn:massart_glm}), we use a similar certificate-based approach.
While it is simple to show reweighting with the inverse margin $|\vec w \cdot \x|^{-1}$ still works in this case (see \Cref{lem:glm_gradient_sketch}), using the bounded reweighting $(|\vec w \cdot \x| + \gamma)^{-1}$ does not. Instead we use a new reweighting of of the form $(|\vw\cdot \x|+\alpha\gamma)$ where $\alpha=O(\epsilon)$ (see \Cref{lem:separating-hyperplane-massart-glm}). Using a similar iterative method as the $\Perspectron$ defined in \eqref{eq:perspectron}, we obtain our sample complexity of $\tO(\gamma^{-2}\epsilon^{-4})$ for learning Massart GLMs. We leave as an open question whether the $\approx \eps^{-2}$ overhead in our GLM learner's sample complexity can be removed. We discuss some natural directions towards this goal in~\Cref{sec:glm_potential}, as well as potential barriers encountered in their implementation.

\subsection{Related work}
\label{sec:related_work}

We briefly survey some additional related works here.
First, a common
worst-case assumption used in statistical learning is that the label noise is adversarial (a.k.a.\ agnostic)
\cite{Haussler92}.  In that setting, significant progress
has been made for learning halfspaces when the underlying distribution satisfies structural assumptions  (e.g., it is Gaussian or log-concave) \cite{KKMS:05,KOS:08,AwasthiBHU15,DKS18a, DKTZ20AgnosticSGD,diakonikolas2022learning}.
For learning halfspaces with a margin, the best-known agnostic results have 
runtime and sample complexity that depend exponentially
on the margin $\gamma$ and/or the accuracy parameter $\eps$ 
\cite{SSS:09,LongServedio:11malicious,diakonikolas19marginopt}.  Another important line of work \cite{DKTZ20,DKKTZ20,zhang2021improved} has focused
on learning halfspaces under Tsybakov noise, which is 
a semi-random noise model that extends Massart noise, but is still easier than the agnostic setting. 
We also note that variants of \Cref{defn:massart_glm} have appeared before: the \emph{generalized Tsybakov low noise condition} of \cite{HKLM20} is a close relative which imposes different noise rates within and outside a margin, and the \emph{strong Massart noise} of \cite{ZhangLC17, DKTZ20} is an instance of \Cref{defn:massart_glm} without the margin restriction. %

Our algorithms rely on the certificate framework developed in \cite{DKTZ20b, ChenKMY20} and the Leaky-ReLU loss that has been
extensively used in prior works on learning with random classification and Massart label noise \cite{Bylander:98a,DiakonikolasGT19,ChenKMY20,dkt_forster}. Our main technical contribution is a new certificate 
that relies on an inverse-margin reweighting scheme and can be
estimated using a single sample.  
Similar, ``inverse-margin'' reweighting schemes have been used for learning general halfspaces \cite{ChenKMY20} and online linear classification \cite{DKTZonline}.  Those results have no implications
for the sample complexity of the problem studied here. Finally, we mention that a local reweighting scheme that is somewhat similar in spirit to ours (but very different in its implementation) was previously employed by \cite{KelnerLLST23}, for a different semi-random statistical learning problem.

\paragraph{Concurrent and independent work.} While preparing this manuscript, we were made aware of concurrent and independent work \cite{DiakonikolasZ24}, who gave essentially the same result as \Cref{theorem:halfspace-massart}, in terms of the sample and computational  complexity for learning Massart halfspaces with a margin.

\subsection{Limitations and open problems}

One interesting open direction is providing improved sample complexity guarantees for more general noise models. For example, the misspecified GLM framework (Definition 3.2, \cite{ChenKMY20}) generalizes \Cref{defn:massart_glm} to include an additional misspecification parameter $\zeta$ such that $\eta(\x) \le \frac{1-|\sigma(\vw^\star\cdot \x)|}{2} + \zeta$. Our approach does not directly apply in this setting, since $\zeta = 0$ is important for our separation oracle result of \Cref{lem:separating-hyperplane-massart-glm}.  
A different interesting generalization of the noise model corresponds to the case where the link function $\sigma$ is unknown, which the \cite{ChenKMY20} algorithm can handle (at a much higher sample complexity). While our algorithms require knowledge of $\sigma$, it would be interesting to explore whether techniques from learning single-index models (e.g., \cite{kalai2009isotron,gollakota2023agnostically, ZWDD24}) can be used to extend our algorithms in this setting.

A clear next step is to design efficient algorithms with sample complexities independent of the margin but still linear in the dimension $d$,\footnote{By standard random projection procedures \cite{ArriagaVempala:99}, the dimension $d$ is comparable to $\gamma^{-2}$ under a $\gamma$-margin assumption, and therefore our sample complexity is nearly-linear in the ``dimension.''} e.g., with an $\approx d \eps^{-2}$ sample complexity.  In prior work \cite{dkt_forster} such an algorithm is given, albeit with a significantly worse $\poly(d\eps^{-1})$ sample complexity.

\section{Preliminaries}
\label{sec:prelims}

We denote vectors in lower-case boldface, and $\norm{\x}$ is the Euclidean norm of $\x \in \R^d$. We use $\B^d$ to denote the unit ball in $\R^d$, i.e.\ $\B^d \defeq \{\x \in \R^d \mid \norm{\x} \le 1\}$. We use $\Proj_{\B^d}(\w) \defeq \min\{1, \frac 1 {\norm{\w}}\} \w$ to denote the Euclidean projection of a vector $\w \in \R^d$ onto $\B^d$. We let $\0_d$ be the all-zeroes vector in dimension $d$. We reserve the overline notation $\bar{\x}$ to denote the unit vector in the direction of $\x$, i.e.\ $\bar{\x} \defeq \frac{\x}{\norm{\x}}$. We let $\sign: \R \to \cube$ be defined so $\sign(t) = 1$ iff $t \ge 0$. We use $\1\{\event\}$ to denote the $0$-$1$ indicator of a random event $\event$, $\pr[\event]$ to denote its probability, and $\E$ to denote the expectation operator. The support of a distribution $D$ is denoted $\supp(D)$, and $[N] \defeq \{i \in \N \mid i \le N\}$.

Throughout the paper we study the problem of learning a binary classifier, given labeled examples from a distribution $D$ over labeled examples $(\x, y) \in \R^d \times \{0, 1\}$, under various models on the distribution to be discussed. We refer to the $\x$-marginal of $D$ by $D_{\x}$, and the conditional distribution of the label $y \mid \x$ by $D_{y}(\x)$. We will primarily be interested in learning halfspace hypotheses, which for $\w \in \R^d$ are the corresponding functions $\hw: \R^d \to \cube$ defined by $\hw(\x)\defeq \sign\Par{\ww\cdot \x}$.
We also denote the \emph{zero-one loss} of a halfspace hypothesis $\hw$ corresponding to $\w \in \R^d$ as follows, when the distribution $D$ over labeled examples is clear from context: $\lzo(\vw)\defeq\pr_{(\x, y) \sim D}\Brack{\hw(\x) \neq y}$. We define the Leaky-ReLU function with parameter $\lambda>0$ as 
\[\ell_{\lambda}(t)\defeq (1-\lambda)\max(0,t)-\lambda\max(0,-t) = \begin{cases}
(1 - \lam) t & t \ge 0 \\
\lam t & t \le 0
\end{cases}.\] 
Given $\vw,\x \in \R^d$ and $y\in \cube{}$, it holds that the (sub)gradient of $\ell_\lambda(-y\vw\cdot\x)$ with respect to $\vw$ is\footnote{We never use differentiability of the LeakyReLU in our proofs, and include this definition of the subgradient only to motivate our algorithms, e.g., in Lemma~\ref{claim:warmup_separating_hyperplane}. For simplicity, we abuse notation and denote the subgradient by $\nabla$.}
\begin{equation}\label{eq:lr_grad}\grad_{\vw}\ell_{\lambda}(-y\vw\cdot\x)=\frac{1}{2}((1-2\lambda)\sign(\vw\cdot\x)-y).\end{equation}
We also provide some brief remarks on how to extend Definition~\ref{defn:massart_half} to more general settings here.

\begin{remark}\label{rem:extensions}
\Cref{defn:massart_half}
extends straightforwardly to the case where $D_{\x}$ has margin $\gamma$ and is supported on a subset of $R \cdot \B^d$ for $R \neq 1$, by rescaling so $R \gets 1$ and $\gamma \gets \frac \gamma R$, as halfspace hypotheses and our label noise assumptions only depend on signs. Other than these margin and support assumptions, we make no additional distributional assumptions about the $\x$-marginal $D_{\x}$. 

Further, due to working in the distributional assumption-free setting, we can assume with up to constant factor loss (in margin) that the halfspace is homogeneous, i.e., has no constant shift term. That is, given a halfspace $\sign(\vw\cdot \x+b)$ with $\norm{\vw},|b|\leq 1$, after a feature expansion ($\x\mapsto \frac 1 {\sqrt 2}(\x,1)$) the halfspace $h_{\vw'}(\z)=\sign((\vw, b)\cdot \z)$ is homogeneous. Moreover, assuming that $|\vw \dot \x + b| \ge \gamma$ for all $\x \in \supp(D_{\x})$, i.e., the distribution has a margin with respect to the (biased) decision boundary, the feature-expanded distribution still has a margin $\geq \frac \gamma 2$ with respect to $\vw' = \frac 1 {\sqrt{1 + b^2}}(\vw,b)$. 

Finally, the Massart noise model of \cite{Massart2006} is 
defined for any hypothesis class and is not tied to halfspaces.  Since we focus on learning halfspaces with a margin, 
we combined the hypothesis class $\{h_{\w}\}_{\w \in \R^d}$ with the noise model in \Cref{defn:massart_half} for simplicity. 
\end{remark}

\section{Massart halfspaces}
\label{sec:halfspace}
In this section, we give our result on learning Massart halfspaces 
with margin. Our proof is rather surprisingly short, and we separate it into its two main components: a structural lemma in \Cref{sec:massart_sep_hyperplanes} which shows how to estimate a separating hyperplane with a single sample given a sub-optimal $\vw$, and a perceptron-like analysis of a stochastic iterative method in \Cref{ssec:sgd}. 

\subsection{Separating hyperplanes for Massart halfspaces}
\label{sec:massart_sep_hyperplanes}

We prove our main structural lemma here, used to argue the progress of our iterative method. As highlighted in \Cref{sec:intro_massart_overview}, we show that when the current $\lzo(\w) \ge \eta + \eps$, we show that we can construct an unbiased estimator for a separating hyperplane between $\w$ and the target vector $\ws$. 

\paragraph{Warmup: an ``unbounded'' separating hyperplane.}
Before presenting the full proof, we first give a separating hyperplane that works for any feature distribution --- even without margin assumptions. The proposed separating hyperplane works because we can express $\lzo$ in terms of a reweighted Leaky-ReLU (a convex function), as was previously observed by \cite{DiakonikolasGT19}. 

\begin{lemma}[Separating hyperplane for Massart halfspaces]
    \label{claim:warmup_separating_hyperplane}
  Let $D$ be an instance of the $\eta$-Massart halfspace model, and $\w \in \R^d$ have classification error $\lzo(\w)  \ge \eta + \eps$.
  It holds that
\[
\E_{(\x, y) \sim D}\Brack{  \frac{\grad_{\vw}\ell_{\eta}(-y\vw\cdot\x)}{|\vec w \cdot \x|} } \cdot (\vec w - \vw^\star)  \geq  \eps\,.
\]
\end{lemma}
\begin{proof}
 We recall Claim~2.1 from \cite{DiakonikolasGT19}, which states that for all $\lam \ge 0$ and $\vw,\x$, it holds that 
 \begin{equation}\label{eq:reweight_lr}\E_{y\sim D_y(\x)}\Brack{\ell_{\lambda}(-y\vw\cdot\x)}=\Par{\pr_{y\sim D_y(\x)}\Brack{\sign(\vw\cdot \x)\neq y}-\lambda}\cdot|\vw\cdot \x|.\end{equation}
 In particular, we have by assumption that
 \begin{equation}\label{eq:signal_warmup}\E_{(\x,y)\sim D}\Brack{\frac{\ell_{\eta}(-y\vw\cdot\x)}{|\vw\cdot\x|}}=\lzo(\vw)-\eta \ge \eps.\end{equation}
 From the convexity of $\ell_{\eta}(-y\vw\cdot\x)$ in $\ww$, we obtain that $\grad_{\vw}\ell_{\eta}(-y\vw\cdot\x)\cdot(\vw-\vw^\star)\geq \ell_{\eta}(-y\vw\cdot\x)-\ell_{\eta}(-y\vw^\star \cdot \x)$. By dividing both sides by $|\vw\cdot\x|$ and taking expectation over $\x$ and $y$, we obtain
\begin{align*}
\E_{(\x,y) \sim D}\Brack{\frac{\grad_{\vw}\ell_{\eta}(-y\vw\cdot\x)} {|\vw\cdot \x|}}\cdot(\vw-\vw^\star)&\geq  \E_{(\x,y) \sim D}\Brack{\frac{\ell_\eta(-y\vw\cdot\x)}{|\vw\cdot \x|}}-  \E_{(\x,y) \sim D}\Brack{\frac{\ell_\eta(-y\ws\cdot\x)}{|\vw\cdot \x|}} \\
&\geq \E_{(\x,y) \sim D}\Brack{\frac{\ell_\eta(-y\vw\cdot\x)}{|\vw\cdot \x|}} \ge \epsilon \,,
\end{align*}
where the last inequality used \Cref{eq:signal_warmup}, and the second-to-last used \Cref{eq:reweight_lr} which implies 
\[\E_{(\x,y) \sim D}\Brack{\ell_\eta(-y\ws\cdot\x)} =\E_{\x \sim \Dx}\Brack{ \Par{\eta(\x) - \eta} \cdot |\vw \cdot \x|} \le 0.\]
\end{proof}

\paragraph{A ``bounded'' separating hyperplane for $\gamma$-margin Massart halfspaces.}
Our claim in the previous lemma was very general: it works for any marginal distribution. However, as discussed in \Cref{sec:intro_massart_overview}, the unbounded nature of this separating hyperplane may make it impossible to estimate from samples. To overcome this, we propose a new candidate hyperplane: \[\E_{\x,y}\Brack{\frac{\grad_{\vw}\ell_{\eta}(-y\vw\cdot\x)}{|\vw\cdot\x|+\gamma}}.\]  
We prove that this is indeed a separating hyperplane by leveraging the margin $\gamma$ with respect to the optimal halfspace $\ws$. Recall from \Cref{eq:lr_grad} that $\grad_{\vw}\ell_\eta(-y\vw\cdot\x)=\frac{1}{2}((1-2\eta)\sign(\vw\cdot\x)-y)$.
\begin{lemma}[Bounded separating hyperplane for Massart halfspaces]
\label{lem:gradient_lb}
Let $D$ be an instance of the $\eta$-Massart halfspace model with margin $\gamma$ and define $\beta = 1-2 \eta$. 
If $\w \in \R^d$ has $\lzo(\w)  \ge \eta + \eps$, it holds that
\[
\E_{(\x, y) \sim D}\Big[ (\beta \sgn(\vec w \cdot \x) - y) \frac{\x}{|\vec w \cdot \x| + \gamma} \Big] \cdot (\vec w - \vw^\star)  \geq 2 \eps\,. 
\]
\end{lemma}
\begin{proof}
We first observe that by the definition of the Massart halfspace  model, 
$\E_{y \sim D_y(\x)}[y] = (1-2 \eta(\x)) \sgn(\vw^\star \cdot \x) = \beta(\x)\sign(\ws \cdot \x)$,
where $\beta(\x) \defeq 1 - 2 \eta(\x)$. 
Therefore, we have that 
\begin{align*}
I &\defeq 
\E_{(\x, y) \sim D}\Big[ (\beta \sgn(\vec w \cdot \x) - y) \frac{
(\vec w \cdot \x - \vw^\star \cdot \x)}{|\vec w \cdot \x| + \gamma} \Big] 
\\
&=
\E_{\x \sim D_\x}\Big[ (\beta \sgn(\vec w \cdot \x) - \beta(\x)   \sgn(\vw^\star \cdot \x)) \frac{
(\vec w \cdot \x - \vw^\star \cdot \x)}{|\vec w \cdot \x| + \gamma} \Big]  \,.
\end{align*}
We denote by $g(\x) \defeq
(\beta \sgn(\vec w \cdot \x) - \beta(\x)  \sgn(\vw^\star \cdot \x)) \frac{
(\vec w \cdot \x - \vw^\star \cdot \x)}{|\vec w \cdot \x| + \gamma} $, which we bound differently based on whether $\x$ falls in the agreement region
$A \defeq \Brace{\x \in \B^d \mid \hws(\x) = \hw(\x)}$.
For $\x \in A$,
\begin{align*}
    g(\x)&=\big(\beta\sign(\vw\cdot \x)-\beta(\x)\sign(\vw^\star\cdot \x)\big)\frac{(\vw\cdot \x-\vw^\star\cdot \x)}{|\vw\cdot \x|+\gamma}\\
    &=\big(\beta-\beta(\x)\big)\frac{|\vw\cdot\x|-|\vw^\star\cdot \x|}{|\vw\cdot \x|+\gamma}\geq \beta-\beta(\x)\,.
\end{align*}
The second equality follows from the fact that $\sign(\vw^\star\cdot \x)=\sign(\vw\cdot \x)$. The final inequality holds since $\beta-\beta(\x)\leq 0$ and $\frac{|\vw\cdot\x|-|\vw^\star\cdot\x|}{|\vw\cdot \x|+\gamma}\leq 1$. Similarly, for $\x\notin A$, an analogous calculation yields
\(g(\x)=\big(\beta+\beta(\x)\big)\frac{|\vw\cdot\x|+|\vw^\star\cdot \x|}{|\vw\cdot \x|+\gamma}\geq \beta+\beta(\x)\,.\)
The first equality holds because $\sign(\vw^\star\cdot \x)\neq \sign(\vw\cdot \x)$ and the final inequality follows since $|\vw^\star\cdot \x|\geq \gamma$ from the margin assumption. Thus 
\begin{align}
I\geq \E_{\x\sim \D_\x}\big[\1\{\x\in A\}(\beta-\beta(\x))\big]+\E_{\x\sim \D_\x}\big[\1\{\x\notin A\}(\beta+\beta(\x))\big].
\label{eq:total-contribution-lower-bound}
\end{align}

We will now use our lower bound on $\lzo(\w)$, which we relate to \Cref{eq:total-contribution-lower-bound}. We have
$ \lzo(\w) = \E_{\x \sim \Dx}\Brack{\1\Brace{\x \in A}\eta(\x)} + \E_{\x \sim \Dx}\Brack{\1\Brace{\x \not\in A}(1 - \eta(\x))} 
= \E_{\x \sim \Dx}\Brack{\1\Brace{\x \not\in A}\beta(\x)} + \E_{x \sim \Dx}\Brack{\eta(\x)}$.
Next, by our definition $\beta(\x) = 1 - 2\eta(\x)$, rearranging and expanding we have:
\begin{equation}\label{eq:optimality-gap-equivalent-form}
\begin{aligned}
\lzo(\w)- \eta &=
\E_{\x \sim \D_\x}[\1\{\x \notin A\} \beta(\x)] +
\frac{1}{2}\E_{\x \sim \D_\x}[\beta-\beta(\x)] 
\\
&= 
\E_{\x \sim \D_\x}[\1\{\x \notin A\} \beta(\x)]
+
\frac{1}{2}\E_{\x \sim \D_\x}[ (\1\{\x \in A\} + \1\{\x \notin A\}) (\beta- \beta(\x))] 
\\
&=
\frac{1}{2}
\E_{\x \sim \D_\x}[\1\{\x \notin A\} (\beta(\x) + \beta)]
+
\frac{1}{2}\E_{\x \sim \D_\x}[ \1\{\x \in A\} (\beta - \beta(\x))] 
\,.
\end{aligned}
\end{equation}

We finish the proof by combining \Cref{eq:total-contribution-lower-bound}, \Cref{eq:optimality-gap-equivalent-form}, and
$\lzo(\w) - \eta \geq \eps $.
\end{proof}
\subsection{{Perspectron}}\label{ssec:sgd}

We now present and analyze $\mathsf{Perspectron}$, our algorithm for learning Massart halfspaces.

\begin{algorithm2e}[H]
	\caption{$\mathsf{Perspectron}$}
	\label{alg:massart_margin}
	\DontPrintSemicolon
		\codeInput $\{\x^i, y^i\}_{i \in [T_1 + T_2]} \subset \R^d \times \cube$ drawn i.i.d.\ from $D$ in the $\eta$-Massart halfspace model with margin $\gamma$, step size $\lam > 0$, failure probability $\delta \in (0, \half)$\;
        $\beta \gets 1 - 2\eta$, $N \gets \lceil \log_2(\frac 2 \delta) \rceil $, $T \gets \lceil \frac{T_1} N\rceil$\;
        $H \gets \emptyset$\;
        \For{$j \in [N]$}{
        $\w^{1, j} \gets \0_d$\;
        \For{$t \in [\min(T, T_1 - (j - 1)T)]$}{\label{line:oneloop_start}
        $i \gets (j - 1)T + t$\;
        $\w^{t + 1, j} \gets\w^{t, j} - \lam \frac{\beta \sign(\w^{t, j} \cdot \x^{i}) - y^{i}}{|\w^{t, j} \cdot \x^i| + \gamma} \x^i$
        }\label{line:oneloop_end}
        $H \gets H \cup \{\w^{t, j}\}_{t \in [\min(T, T_1 - (j - 1)T)]}$\;
        }
        $S \gets \{\x^i, y^i\}_{i = T_1 + 1}^{T_1 + T_2}$\;
        $\w \gets \arg\min_{\w \in H} \pr_{(\x, y) \simu S}[h_{\w}(\x) \neq y]$\;\label{line:selection}
        \codeReturn $h_{\w}$\;
\end{algorithm2e}
We begin by giving a self-contained analysis of a single loop $j \in [N]$ of \Cref{line:oneloop_start} to \Cref{line:oneloop_end}, showing that for sufficiently large $T$, at least one iterate achieves small $\lzo$ with constant probability.

\begin{lemma}\label{lem:sgd}
Let $\{\x^i, y^i\}_{i \in [T]} \sim_{i.i.d.} D$, where $D$ is an instance of the $\eta$-Massart halfspace model. Consider iterating, from $\w^1 \defeq \0_d$,
\begin{equation}\label{eq:halfspace_iter}\w^{t + 1} \gets \w^t - \lam \frac{\beta \sign(\w^t \cdot \x^t) - y^t}{|\w^t \cdot \x^t| + \gamma}\x^t,\end{equation}
for $\beta \defeq 1 - 2\eta$, $\lam \defeq \frac{\gamma}{2\sqrt{T}}$. Then if $T \ge \frac{16}{\eps^2 \gamma^2}$,
$\pr[\min_{t \in [T]} \lzo(\w^t) \ge \eta + \frac \eps 2] \le \half$.
\end{lemma}
 \begin{proof}
 Throughout the proof, say $\w \in \R^d$ is \emph{bad} iff $\lzo(\w) \ge \eta + \frac \eps 2$, and let $\event_t$ denote the event that all of the iterates $\{\w^s\}_{s \in [t]}$ updated according to \Cref{eq:halfspace_iter} are bad. 
 
 Define the potential function
\(\Phi_t \defeq \E[\1\Brace{\event_t}\cdot\norm{\ws - \w^t}^2] \text{ for all } t \in [T].\)
On expanding the expression for the squared norm and using the fact that $\1\Brace{\event_{t + 1}} \le \1\Brace{\event_{t}}$,
\begin{align*}
\Phi_{t + 1} &\le \E\Brack{\1\Brace{\event_t} \cdot \norm{\ws - \w^{t + 1}}^2} \\
&\le \Phi_t + \lam^2 \E\Brack{\norm{ \frac{\beta\sign(\w^t \cdot \x^t) - y^t}{|\w^t \cdot \x^t| + \gamma} \x^t}^2} \\
&- 2\lam\E\Brack{\1\Brace{\event_t} \cdot \frac{\beta\sign(\w^t \cdot \x^t) - y^t}{|\w^t \cdot \x^t| + \gamma}\x\cdot({\w^t - \ws})} \\
&\le \Phi_t + \frac{4\lam^2}{\gamma^2} - 2\lam \pr\Brack{\event_t} \E\Brack{\frac{\beta\sign(\w^t \cdot \x^t) - y^t}{|\w^t \cdot \x^t| + \gamma}{\x^t}\cdot({\w^t - \ws}) \mid \event_t} \\
&\le \Phi_t + \frac{4\lam^2}{\gamma^2} - 2\lam\eps \pr\Brack{\event_t}.
\end{align*}
Here, the third inequality used $\x^t \in \B^d$ and $|\beta\sign(\w^t \cdot \x^t) - y^t| \le 2$, and the fourth applied \Cref{lem:gradient_lb}. Now using that $\Phi_1 \le \norm{\ws}^2 = 1$, $\Phi_{T + 1} \ge 0$, and $\Pr[\event_t] \ge \Pr[\event_T]$ for $t \in [T]$, we have 
\[2\lam\eps T \Pr\Brack{\event_T} \le 1 + \frac{4\lam^2 T}{\gamma^2}.\]
The conclusion $\Pr\Brack{\event_T} \le \half$ follows from our choices of $\lam$, $T$.
\end{proof}

We next analyze the hypothesis selection step in \Cref{line:selection}.
\begin{lemma}[Hypothesis selection]\label{lem:selection}
Suppose there exists $\hat{\w} \in H$ with $\lzo(\w) \le \eta + \frac \eps 2$. Then if $T_2 \ge \frac 8{\eps^2}\log(\frac{2|H|}{\delta})$, with probability $\ge 1 - \delta$ the $\w$ returned by \Cref{line:selection} satisfies $\lzo(\w) \le \eta + \eps$.
\end{lemma}
\begin{proof}
Because $S$ is independent of $H$, Hoeffding's inequality and a union bound implies that $\Abs{\pr_{(\x, y) \simu S}[h_{\w}(\x) \neq y] - \lzo(\w)} \le \frac \eps 4 \text{ with probability } \ge 1 - \delta$,
for all $\w \in H$. Conditioning on this event, we have the claim from:
\begin{align*}
\lzo(\w) - \frac \eps 4 &\le \pr_{(\x, y) \simu S}[h_{\w}(\x) \neq y] 
\le \pr_{(\x, y) \simu S}[h_{\hat{\w}}(\x) \neq y] \le \lzo(\hat{\w}) + \frac \eps 4 \le \eta + \frac{3\eps}{4}.\qedhere
\end{align*}
\end{proof}

We are now ready to state and prove our main theorem. 
\begin{theorem}[Learning Massart halfspaces with margin]\label{theorem:halfspace-massart}
Let $D$ be an instance of the $\eta$-Massart halfspace model with margin $\gamma$, and let $\eps, \delta \in (0, 1)$. \Cref{alg:massart_margin} with $T_1 \ge \frac{16}{\eps^2\gamma^2}\lceil\log_2(\frac 2 \delta) \rceil$, $T_2 \ge \frac{8}{\eps^2}\log(\frac{4|T_1|}{\delta})$ returns $\w$ such that $\lzo(\w) \le \eta + \eps$ with probability $\ge 1 - \delta$, using
\begin{align*}
O\Par{\frac{\log \frac 1 \delta}{\eps^2 \gamma^2} + \frac{\log \frac{1}{\eps\gamma\delta}}{\eps^2}}\text{ samples and } O\Par{\frac{d\log(\frac 1 \delta)\log(\frac 1 {\eps\gamma\delta})}{\eps^4\gamma^2}} \text{ time.}
\end{align*}
\end{theorem}
\begin{proof}
First, applying \Cref{lem:sgd} to each of the $N$ independent runs of \Cref{line:oneloop_start} to \Cref{line:oneloop_end} shows that the premise of \Cref{lem:selection} is met except with probability $\frac \delta 2$. The correctness claim then follows from \Cref{lem:selection}. The sample complexity is immediate, and the runtime bound follows because the bottleneck operation is computing the value of $h_{\w}(\x)$ for all $(\x, y) \in S$ and $\w \in H$.
\end{proof}

\section{Massart GLMs}
\label{sec:glm}

In this section, we prove our result regarding learning in the Massart generalized linear model with a margin. Our analysis is similar to that of \Cref{sec:halfspace} but requires a modified version of \Cref{lem:gradient_lb}, which gives a ``bounded'' separating hyperplane for Massart GLMs (\Cref{lem:separating-hyperplane-massart-glm}). We first restate our model, slightly generalized to relax the assumption that $\sigma$ is odd. 
\begin{definition}[Massart GLM]
\label{defn:massart_glm_appendix}
Let $\sigma:[-1, 1]\to [-1,1]$ be a non-decreasing function. We say that a distribution $D$ on $\B^d \times \cube$ is an instance of the \emph{$\sigma$-Massart generalized linear model (GLM) with constant shift $\tau$ and margin $\gamma$} if the following conditions hold.
\begin{itemize}
    \item $\bigr||\sigma(t)| - |\sigma(-t)|\bigr| \le \tau$ for all $t \in [0, 1]$.
    \item There exists $\ws \in \R^d$ such that $\norm{\ws} = 1$ and $\pr[|\ws\cdot \x|<\gamma]=0$.
    \item For all $\x \in \supp(\Dx)$, there is an $\eta(\x) \in [0, \frac{1 - |\sigma(\ws \cdot \x)|}{2}]$ such that
\[\pr_{y \sim D_y(\x)}\Brack{y \neq \hws(\x)} = \eta(\x).\]
\end{itemize}
\end{definition}

In the setting of \Cref{defn:massart_glm_appendix}, instead of being upper bounded by a fixed constant $\eta$, the noise rate is data-dependent and upper bounded by $\frac{1-\sigma(\ws \cdot \x)}{2}$ where $\sigma$ is non-decreasing and $\ws$ is the optimal halfspace. Moreover, $\sigma$'s deviation from being odd is governed by the parameter $\tau$. 

Inspired by our approach in \Cref{sec:halfspace}, here we begin by proposing a novel separating hyperplane based on our previously-described reweighting scheme. We first argue that $\E_{\x,y}\big[\frac{(\sigma(\vw\cdot\x)-y)}{|\vw\cdot\x|}\x\big]$ is a valid separating hyperplane, generalizing \Cref{claim:warmup_separating_hyperplane}.

\begin{lemma}[Separating hyperplane for Massart GLMs]
    \label{lem:glm_gradient_sketch}
     Let $D$ be an instance of the $\sigma$-Massart GLM (Definition~\ref{defn:massart_glm_appendix}), and $\vw\in \R^{d}$ have $\lzo(\vw)\geq \optrcn+\epsilon$. It holds that
\[
\E_{(\x,y)\sim D}\Brack{\frac{(\sigma(\vw\cdot\x)-y)}{|\vw\cdot \x|}\x}
\cdot (\vw-\vw^\star)\geq 2\epsilon\,.
\]
\end{lemma}
\begin{proof}
We use the definition of the set $A$ and the signal level $\beta(\x)$ from \Cref{lem:gradient_lb}.
By expanding 
\[\lzo(\vw) = \E_{\x \sim D_{\x}}\Brack{\1\{\x \in A\}\eta(\x)} + \E_{\x \sim D_{\x}}\Brack{\1\{\x \not\in A\}(1 - \eta(\x))}\] 
similarly to \Cref{eq:optimality-gap-equivalent-form}, we obtain that 
\begin{equation}\label{eq:glm_equiv}\frac{1}{2}\cdot \E_{(\x,y)\sim D}\left[(|\sigma(\vw^\star\cdot \x)|-\beta(\x))\1\{\x\in A\}\right]+\frac{1}{2}\cdot\E_{(\x,y)\sim D}\left[(|\sigma(\vw^\star\cdot \x)|+\beta(\x))\1\{\x\notin A\}\right]\geq \epsilon.\end{equation} Let $g(\x)=\frac{(\sigma(\vw\cdot \x)-\beta(\x))}{|\vw\cdot\x|}\cdot(\vw-\vw^\star)$. Our goal is to argue that $\half \E_{\x \sim D_{\x}}[g(\x)]$ is at least the left-hand side of \Cref{eq:glm_equiv}, which immediately gives the claim. We do so via case analysis. 

For $\x\in A$, we observe that 
\[g(\x)\geq(|\sigma(\vw^\star\cdot\x)|-\beta(\x))\frac{|\vw\cdot\x|-|\vw^\star\cdot\x|}{|\vw\cdot\x|}\geq |\sigma(\vw^\star\cdot\x)|-\beta(\x).\]
Here, we obtained the first inequality by adding and subtracting the corresponding term with $\sigma(\vw^\star\cdot\x)$ and then using monotonicity. The final inequality follows from $\beta(\x)\geq |\sigma(\vw^\star\cdot\x)|$. 

For $\x\notin A$, we similarly obtain 

\[g(\x)=(|\sigma(\vw\cdot\x)|+\beta(\x))\frac{|\vw\cdot\x|+|\vw^\star\cdot\x|}{|\vw\cdot\x|}\geq |\sigma(\vw^\star\cdot\x)|+\beta(\x).\] 
Here the final inequality follows via another case analysis. If it is the case that $|\vw\cdot\x|\leq |\vw^\star\cdot\x|$, then $\frac{|\vw\cdot\x|+|\vw^\star\cdot\x|}{|\vw\cdot\x|}\geq 2$ and $2\beta(\x)\geq (|\sigma(\vw^\star\cdot\x)|+\beta(\x))$. Otherwise, we have $|\vw\cdot\x|\geq |\vw^\star\cdot\x|$; in this case $|\sigma(\vw\cdot\x)|\geq |\sigma(\vw^\star\cdot\x)|$ and hence we are done. 
\end{proof}

However, we are met with the same obstacle as before: $|\vw\cdot\x|$ can be arbitrarily small. The previous approach of adding $\gamma$ to the denominator does not work immediately. Instead, we add the rescaled term $\frac{\eps}{2-\eps}\cdot \gamma$. Adding a smaller term in the denominator increases the bound on the norm of the increments in each step, resulting to a larger bound on the number of iterations (and, therefore, sample complexity) by a factor of $\eps^{-2}$. However, this rescaling is useful to obtain an analogue of \Cref{lem:gradient_lb} for the case of Massart GLMs (\Cref{lem:separating-hyperplane-massart-glm}). The rescaling essentially accounts for the part of the distribution where $|\w\cdot \x|$ is smaller than $|\w^\star\cdot \x|$ and the signs disagree. This is important because the size of $|\w\cdot \x|$ is quantitatively more significant in the Massart GLM case.

Combining this with a modified version of the Perspectron algorithm and analysis (see \Cref{alg:massart_glm_margin}), we obtain our final result with sample complexity $\wt{O}(\gamma^{-2}\epsilon^{-4})$ (see \Cref{theorem:glm-massart}). 

We state and prove the following lemma which provides a separating hyperplane in this setting.

\begin{lemma}[Bounded separating hyperplane for Massart GLMs]\label{lem:separating-hyperplane-massart-glm}
Let $D$ be an instance of the $\sigma$-Massart GLM with constant shift $\tau$ and margin $\gamma$ (\Cref{defn:massart_glm_appendix}). Let $\vw\in \R^d$ be such that 
\[\Pr_{(\x,y)\sim D}[\sign(\vw\cdot \x)\neq y]\geq \E_{\x\sim D_{\x}}\Brack{\frac{1-|\sigma(\vw^\star\cdot x)|}{2}}+ \frac \tau 2 + \epsilon.\]
It holds that
\[
\E_{(\x,y)\sim D}\left[(\sigma(\vw.\x)-y)\frac{\x}{|\vw\cdot \x|+\alpha\gamma}\right]\cdot (\vw-\vw^\star)\geq \epsilon\,,\,\text{ for $\alpha \defeq \frac \eps {2-\epsilon}$}.
\] 
\end{lemma}

\begin{proof}
    Let $A \defeq \{\x\in \R^d\mid h_{\vw^\star}(\x) = h_{\vw}(\x)\}$. We proceed similarly to \Cref{lem:gradient_lb}: by expanding,
    \begin{align*}
    I&\defeq\E_{(\x,y)\sim D}\left[(\sigma(\vw\cdot\x)-y)\frac{\x}{|\vw\cdot \x|+\alpha\gamma}\right]\cdot (\vw-\vw^\star)\\
    &=\E_{\x\sim D_{\x}}\left[(\sigma(\vw\cdot \x)-\beta(\x)\sign(\vw^\star\cdot \x))\cdot\frac{(\vw\cdot \x-\vw^\star\cdot \x)}{|\vw\cdot \x|+\alpha\gamma}\right]
    \end{align*}
    Define $g(\x)=(\sigma(\vw\cdot \x)-\beta(\x)\sign(\vw^\star\cdot \x))\cdot \frac{(\vw\cdot \x-\vw^\star\cdot \x)}{|\vw\cdot \x|+\alpha\gamma}$. We proceed by casework on $\x \in A$ or $\x \not\in A$.
    
    First, for any $\x\in A$, we have that 
    \begin{align*}
        g(\x)&=\big(\sigma(\vw\cdot\x)-\sigma(\vw^\star\cdot \x)+\sigma(\vw^\star\cdot \x)-\beta(\x)\sign(\vw^\star\cdot \x)\big)\cdot\frac{(\vw\cdot \x-\vw^\star\cdot \x)}{|\vw\cdot \x|+\alpha\gamma}\\
        &\geq \big(|\sigma(\vw^\star\cdot \x)|\sign(\vw^\star\cdot \x)-\beta(\x)\sign(\vw^\star\cdot \x)\big)\cdot\frac{(\vw\cdot \x-\vw^\star\cdot \x)}{|\vw\cdot \x|+\alpha\gamma}\\
        &\geq \big(|\sigma(\vw^\star\cdot \x)|-\beta(\x)\big)\cdot \frac{|\vw\cdot \x|-|\vw^\star\cdot \x|}{|\vw\cdot \x|+\alpha\gamma}\geq|\sigma(\vw^\star\cdot \x)|-\beta(\x)\,. 
    \end{align*}
    The second inequality follows from the fact that $(\sigma(\vw\cdot\x)-\sigma(\vw^\star\cdot \x))\cdot(\vw\cdot \x-\vw^\star\cdot\x)\geq 0$ since $\sigma$ is monotonically increasing. The third inequality holds because $\sign(\vw\cdot \x)=\sign(\vw^\star\cdot \x)$. The final inequality is true because $\frac{|\vw\cdot \x|-|\vw^\star\cdot \x|}{|\vw\cdot \x|+\alpha\gamma}\leq 1$ and $\beta(\x)\geq |\sigma(\vw^\star\cdot \x)|$.

    We now consider the case where $\x\not\in A$. Since $\sign(\vw\cdot \x)\neq \sign(\vw^\star\cdot \x)$, we have that \[g(\x)=\big(|\sigma(\vw\cdot \x)|+\beta(\x)\big)\cdot\frac{|\vw\cdot \x|+|\vw^\star\cdot \x|}{|\vw\cdot \x|+\alpha\gamma}.\] 
    We split the complement of $A$ into two finer regions. Define $B_1=\{\x\not\in A\mid |\vw\cdot \x|\geq |\vw^\star\cdot \x|\}$ and $B_2=\{\x\not\in A\mid |\vw\cdot \x|< |\vw^\star\cdot \x|\}$. We conclude via another casework.
    
    First, consider $\x\in B_1$. We have that $|\sigma(\vw\cdot \x)|\geq \max(0,|\sigma(-\vw\cdot \x)|-\tau)$. We also have that $|\sigma(-\vw\cdot \x|)\geq |\sigma(\vw^\star\cdot \x)|$ since $|\vw\cdot \x|\geq |\vw^\star\cdot \x|$ and $\sigma$ is monotone non-decreasing. Also, observe that $\frac{|\vw\cdot \x|+|\vw^\star\cdot \x|}{|\vw\cdot \x|+\alpha\gamma}\geq 1$ since $|\vw^\star\cdot \x|\geq \gamma$. Thus, we obtain that in this case, 
    \[g(\x)\geq \max(0,|\sigma(\vw^\star\cdot \x)|-\tau)+\beta(\x) \ge \max(0,|\sigma(\vw^\star\cdot \x)|-\tau)+\beta(\x) - \eps.\]
    Finally, we consider $\x\in B_2$. Let $c(\x) \defeq \frac{|\vw^\star\cdot \x|}{|\vw\cdot \x|}$ where $c(\x) > 1$ for $\x \in B_2$. We have that 
    \begin{equation}\label{eqn:advantage_glm}
    \begin{aligned}
        g(\x) &\geq \beta(\x) \cdot \frac{|\vw^\star\cdot \x|+|\vw\cdot\x|}{|\vw\cdot \x|+\alpha\gamma}
         = \beta(\x) \cdot \frac{1+\frac{|\vw^\star\cdot \x|}{|\vw\cdot \x|}}{1+\alpha\frac{\gamma}{|\vw^\star\cdot \x|}\frac{|\vw^\star\cdot \x|}{|\vw\cdot \x|}} \\
         &\geq \beta(\x) \cdot \frac{1+c(\x)}{1+\alpha c(\x)}\geq (2-\epsilon)\beta(\x)\,.
         \end{aligned}
    \end{equation}
    The second inequality follows from the fact that $|\vw^\star\cdot \x|\geq \gamma$, and the last inequality is true because $\frac{1+c}{1+\alpha c}\geq 2-\epsilon$ for any $c\geq 1$ when $\alpha=\frac{\epsilon}{2-\epsilon}$. Since $1\geq \beta(\x)\geq |\sigma(\vw^\star\cdot x)|$, we again have that 
    \[g(\x)\geq \beta(\x)+|\sigma(\vw^\star\cdot \x)|-\epsilon \ge \max(0,|\sigma(\vw^\star\cdot \x)|-\tau)+\beta(\x) - \eps.\]

Thus, combining these two cases, we obtain that
    \begin{align}
    I &\geq \E_{(\x,y)\sim D}\Brack{\big(|\sigma(\vw^\star\cdot \x)|-\beta(\x)\big)\1\{\x\in A\}}\nonumber \\
    &+\E_{(\x,y)\sim D}\Brack{\big(\max(0,|\sigma(\vw^\star\cdot \x)|-\tau)+\beta(\x)\big)\1\{\x\not\in A\}}-\epsilon. \label{eqn:glm_massart_step}
    \end{align}

    We now use our assumption on the error of $\vw$. From an analogous derivation to \Cref{eq:glm_equiv},
    \begin{align*}
        \epsilon&\leq \Pr_{(\x,y)\sim D}[\sign(\vw\cdot \x)\neq y]-\E_{\x\sim D_{\x}}\left[\frac{1-|\sigma(\vw^\star\cdot \x)|}{2}\right] - \frac \tau 2 \\
        &\le \half \cdot \E_{(\x,y)\sim D}\left[\big(|\sigma(\vw^\star\cdot \x)|-\beta(\x)\big)\1\{\x\in A\}\right] \\
        &+ \half \cdot \E_{(\x,y)\sim D}\left[\big(|\sigma(\vw^\star\cdot \x)| - \tau +\beta(\x)\big)\1\{\x\not\in A\}\right] \\
        &\le \half \cdot \E_{(\x,y)\sim D}\left[\big(|\sigma(\vw^\star\cdot \x)|-\beta(\x)\big)\1\{\x\in A\}\right] \\
        &+ \half \cdot \E_{(\x,y)\sim D}\left[\big(\max(0, |\sigma(\vw^\star\cdot \x)| - \tau) +\beta(\x)\big)\1\{\x\not\in A\}\right].
    \end{align*}
    Plugging this into \Cref{eqn:glm_massart_step}, we obtain that $I\geq 2\epsilon-\epsilon\geq \epsilon$. This completes the proof.
\end{proof}

We can now prove our main theorem about Massart GLMs. The algorithm we use, \Cref{alg:massart_glm_margin}, is a slightly-modified version of the Perspectron algorithm (\Cref{alg:massart_margin}), where we substitute the value of the parameter $\gamma$ with $\gamma\cdot \frac{\eps}{2-\eps}$ and the updates involve the function $\sigma$.

\begin{theorem}\label{theorem:glm-massart}
Let $D$ be an instance of the $\sigma$-Massart GLM with constant shift $\tau$ and margin $\gamma$ (\Cref{defn:massart_glm_appendix}), and let $\eps, \delta \in (0, 1)$. \Cref{alg:massart_glm_margin} with $T_1 \ge \frac{32}{\eps^4\gamma^2}\lceil\log_2(\frac 2 \delta) \rceil$, $T_2 \ge \frac{8}{\eps^2}\log(\frac{4|T_1|}{\delta})$ returns $\w$ such that $\lzo(\w) \le \E_{\x\sim D_{\x}}\Brack{\frac{1-|\sigma(\vw^\star\cdot \x|}{2}} + \frac \tau 2 + \eps$ with probability $\ge 1 - \delta$, using
\[O\Par{\frac{\log(\frac 1 \delta)}{\eps^4 \gamma^2} + \frac{\log(\frac 1 {\eps\gamma\delta})}{\eps^2}}\text{ samples and } O\Par{\frac{d\log(\frac 1 \delta)\log(\frac 1 {\eps\delta})}{\eps^6\gamma^2}} \text{ time.}\]
\end{theorem}

\begin{algorithm2e}[H]
	\caption{$\mathsf{GLM Perspectron}$}
	\label{alg:massart_glm_margin}
	\DontPrintSemicolon
		\codeInput $\{\x^i, y^i\}_{i \in [T_1 + T_2]} \subset \R^d \times \cube$ drawn i.i.d.\ from $D$ in the $\sigma$-Massart GLM model with margin $\gamma$, parameter $\alpha\in(0,1)$, step size $\lam > 0$, failure probability $\delta \in (0, \half)$\;
        $\beta \gets 1 - 2\eta$, $N \gets \lceil \log_2(\frac 2 \delta) \rceil $, $T \gets \lceil \frac{T_1} N\rceil$, $\alpha \gets \frac{\eps}{2-\eps}$\;
        $H \gets \emptyset$\;
        \For{$j \in [N]$}{
        $\w^{1, j} \gets \0_d$\;
        \For{$t \in [\min(T, T_1 - (j - 1)T)]$}{\label{line:oneloop_start_GLM}
        $i \gets (j - 1)T + t$\;
        $\w^{t + 1, j} \gets\w^{t, j} - \lam \frac{\sigma(\w^{t, j} \cdot \x^{i}) - y^{i}}{|\w^{t, j} \cdot \x^i| + \alpha\gamma} \x^i$
        }\label{line:oneloop_end_GLM}
        $H \gets H \cup \{\w^{t, j}\}_{t \in [\min(T, T_1 - (j - 1)T)]}$\;
        }
        $S \gets \{\x^i, y^i\}_{i = T_1 + 1}^{T_1 + T_2}$\;
        $\w \gets \arg\min_{\w \in H} \pr_{(\x, y) \simu S}[h_{\w}(\x) \neq y]$\;\label{line:selection_GLM}
        \codeReturn $h_{\w}$\;
\end{algorithm2e}

\begin{proof}
    Given \Cref{lem:separating-hyperplane-massart-glm}, the first step of the proof is exactly analogous to the proof of \Cref{lem:sgd}, i.e., we can show the following claim by using \Cref{lem:separating-hyperplane-massart-glm} in place of \Cref{lem:gradient_lb}.
    \begin{claim}
        Let $\{\x^i, y^i\}_{i \in [T]} \sim_{i.i.d.} D$, where $D$ is an instance of the $\sigma$-Massart GLM with constant shift $\tau$ and margin $\gamma$ (\Cref{defn:massart_glm_appendix}). Consider iterating, from $\w^1 \defeq \0_d$,
        \begin{equation}\label{eq:GLM_iter}\w^{t + 1} \gets \w^t - \lam \frac{\sigma(\w^t \cdot \x^t) - y^t}{|\w^t \cdot \x^t| + \gamma \eps/(2-\eps)}\x^t,\end{equation}
        for $\lam \defeq \frac{\gamma\eps}{(2-\eps)\sqrt{2T}}$. Then if $T \ge \frac{32}{\eps^4 \gamma^2}$,
        $\pr[\min_{t \in [T]} \lzo(\w^t) \ge \E_{\x\sim D_{\x}}[\frac{1-|\sigma(\vw^\star\cdot \x|}{2}] + \frac {\tau + \eps} 2] \le \half$.
        \end{claim}
        The rest of the proof is analogous to that of \Cref{lem:sgd}, but we use $\alpha\gamma$ in place of $\gamma$. To amplify the success probability and finish the proof, we once more use \Cref{lem:selection}.
\end{proof}

\section*{Acknowledgements}

We thank Ilias Diakonikolas and Nikos Zarifis for coordinating arXiv and NeurIPS submissions with regards to their concurrent and independent work, \cite{DiakonikolasZ24}.

\newpage

\bibliographystyle{alpha}
\bibliography{ref}

\newcommand{\etalchar}[1]{$^{#1}$}
\begin{thebibliography}{DKTZ20b}

\bibitem[ABHU15]{AwasthiBHU15}
Pranjal Awasthi, Maria{-}Florina Balcan, Nika Haghtalab, and Ruth Urner.
\newblock Efficient learning of linear separators under bounded noise.
\newblock In {\em Proceedings of The 28th Conference on Learning Theory, {COLT} 2015}, volume~40 of {\em {JMLR} Workshop and Conference Proceedings}, pages 167--190. JMLR.org, 2015.

\bibitem[AL88]{AngluinL88}
D.~Angluin and P.~Laird.
\newblock Learning from noisy examples.
\newblock {\em Mach. Learn.}, 2(4):343--370, 1988.

\bibitem[AV99]{ArriagaVempala:99}
R.~Arriaga and S.~Vempala.
\newblock An algorithmic theory of learning: Robust concepts and random projection.
\newblock In {\em Proceedings of the 40\textsuperscript{th} Annual Symposium on Foundations of Computer Science (FOCS)}, pages 616--623, New York, NY, 1999.

\bibitem[BFKV98]{BlumFKV98}
Avrim Blum, Alan~M. Frieze, Ravi Kannan, and Santosh~S. Vempala.
\newblock A polynomial-time algorithm for learning noisy linear threshold functions.
\newblock {\em Algorithmica}, 22(1/2):35--52, 1998.

\bibitem[BGL{\etalchar{+}}24]{BlumGLMSY24}
Avrim Blum, Meghal Gupta, Gene Li, Naren~Sarayu Manoj, Aadirupa Saha, and Yuanyuan Yang.
\newblock Dueling optimization with a monotone adversary.
\newblock In {\em International Conference on Algorithmic Learning Theory}, volume 237 of {\em Proceedings of Machine Learning Research}, pages 221--243. {PMLR}, 2024.

\bibitem[Blu03]{Blum03}
A.~Blum.
\newblock Machine learning: My favorite results, directions, and open problems.
\newblock In {\em 44\textsuperscript{th} Symposium on Foundations of Computer Science {(FOCS} 2003)}, pages 11--14, 2003.

\bibitem[BS95]{BlumS95}
Avrim Blum and Joel Spencer.
\newblock Coloring random and semi-random k-colorable graphs.
\newblock {\em J. Algorithms}, 19(2):204--234, 1995.

\bibitem[Byl94]{Bylander94}
Tom Bylander.
\newblock Learning linear threshold functions in the presence of classification noise.
\newblock In {\em Proceedings of the Seventh Annual {ACM} Conference on Computational Learning Theory, {COLT} 1994}, pages 340--347. {ACM}, 1994.

\bibitem[Byl98]{Bylander:98a}
T.~Bylander.
\newblock Worst-case analysis of the {P}erceptron and exponentiated update algorithms.
\newblock {\em Artificial Intelligence}, 106, 1998.

\bibitem[CG18]{ChengG18}
Yu~Cheng and Rong Ge.
\newblock Non-convex matrix completion against a semi-random adversary.
\newblock In {\em Conference On Learning Theory, {COLT} 2018}, volume~75 of {\em Proceedings of Machine Learning Research}, pages 1362--1394. {PMLR}, 2018.

\bibitem[CKMY20]{ChenKMY20}
Sitan Chen, Frederic Koehler, Ankur Moitra, and Morris Yau.
\newblock Classification under misspecification: Halfspaces, generalized linear models, and evolvability.
\newblock In {\em Advances in Neural Information Processing Systems 33: Annual Conference on Neural Information Processing Systems 2020}, 2020.

\bibitem[CN08]{CN08}
R.~M. Castro and R.~D. Nowak.
\newblock Minimax bounds for active learning.
\newblock {\em IEEE Transactions on Information Theory}, 54(5):2339--2353, 2008.

\bibitem[Coh97]{Cohen97}
Edith Cohen.
\newblock Learning noisy perceptrons by a perceptron in polynomial time.
\newblock In {\em 38th Annual Symposium on Foundations of Computer Science, {FOCS} '97}, pages 514--523. {IEEE} Computer Society, 1997.

\bibitem[Dan16]{Daniely16}
Amit Daniely.
\newblock Complexity theoretic limitations on learning halfspaces.
\newblock In {\em Proceedings of the 48th Annual {ACM} {SIGACT} Symposium on Theory of Computing, {STOC} 2016}, pages 105--117. {ACM}, 2016.

\bibitem[DDK{\etalchar{+}}23]{DiakonikolasDKWZ23}
Ilias Diakonikolas, Jelena Diakonikolas, Daniel~M. Kane, Puqian Wang, and Nikos Zarifis.
\newblock Information-computation tradeoffs for learning margin halfspaces with random classification noise.
\newblock In {\em The Thirty Sixth Annual Conference on Learning Theory, {COLT} 2023}, volume 195 of {\em Proceedings of Machine Learning Research}, pages 2211--2239. {PMLR}, 2023.

\bibitem[DGT19]{DiakonikolasGT19}
Ilias Diakonikolas, Themis Gouleakis, and Christos Tzamos.
\newblock Distribution-independent {PAC} learning of halfspaces with massart noise.
\newblock In {\em Advances in Neural Information Processing Systems 32: Annual Conference on Neural Information Processing Systems 2019}, pages 4751--4762, 2019.

\bibitem[DK20]{DK20-SQ-Massart}
I.~Diakonikolas and D.~M. Kane.
\newblock Hardness of learning halfspaces with massart noise.
\newblock {\em CoRR}, abs/2012.09720, 2020.

\bibitem[DKK{\etalchar{+}}20]{DKKTZ20}
I.~Diakonikolas, D.~M. Kane, V.~Kontonis, C.~Tzamos, and N.~Zarifis.
\newblock A polynomial time algorithm for learning halfspaces with tsybakov noise.
\newblock {\em arXiv}, 2020.

\bibitem[DKK{\etalchar{+}}22]{DKKTZGeneral}
Ilias Diakonikolas, Daniel~M Kane, Vasilis Kontonis, Christos Tzamos, and Nikos Zarifis.
\newblock Learning general halfspaces with general massart noise under the gaussian distribution.
\newblock In {\em Symposium on Theory of Computation}, volume~54, 2022.

\bibitem[DKM19]{diakonikolas19marginopt}
Ilias Diakonikolas, Daniel Kane, and Pasin Manurangsi.
\newblock Nearly tight bounds for robust proper learning of halfspaces with a margin.
\newblock In H.~Wallach, H.~Larochelle, A.~Beygelzimer, F.~d\textquotesingle Alch\'{e}-Buc, E.~Fox, and R.~Garnett, editors, {\em Advances in Neural Information Processing Systems}, volume~32. Curran Associates, Inc., 2019.

\bibitem[DKS18]{DKS18a}
I.~Diakonikolas, D.~M. Kane, and A.~Stewart.
\newblock Learning geometric concepts with nasty noise.
\newblock In {\em Proceedings of the 50\textsuperscript{th} Annual {ACM} {SIGACT} Symposium on Theory of Computing, {STOC} 2018}, pages 1061--1073, 2018.

\bibitem[DKT21]{dkt_forster}
Ilias Diakonikolas, Daniel Kane, and Christos Tzamos.
\newblock Forster decomposition and learning halfspaces with noise.
\newblock In M.~Ranzato, A.~Beygelzimer, Y.~Dauphin, P.S. Liang, and J.~Wortman Vaughan, editors, {\em Advances in Neural Information Processing Systems}, volume~34, pages 7732--7744. Curran Associates, Inc., 2021.

\bibitem[DKTZ20a]{DKTZ20}
I.~Diakonikolas, V.~Kontonis, C.~Tzamos, and N.~Zarifis.
\newblock Learning halfspaces with massart noise under structured distributions.
\newblock In {\em Conference on Learning Theory, {COLT}}, 2020.

\bibitem[DKTZ20b]{DKTZ20b}
I.~Diakonikolas, V.~Kontonis, C.~Tzamos, and N.~Zarifis.
\newblock Learning halfspaces with tsybakov noise.
\newblock {\em arXiv}, 2020.

\bibitem[DKTZ20c]{DKTZ20AgnosticSGD}
I.~Diakonikolas, V.~Kontonis, C.~Tzamos, and N.~Zarifis.
\newblock Non-convex {SGD} learns halfspaces with adversarial label noise.
\newblock In {\em Advances in Neural Information Processing Systems, {NeurIPS}}, 2020.

\bibitem[DKTZ22]{diakonikolas2022learning}
Ilias Diakonikolas, Vasilis Kontonis, Christos Tzamos, and Nikos Zarifis.
\newblock Learning general halfspaces with adversarial label noise via online gradient descent.
\newblock In {\em International Conference on Machine Learning}, pages 5118--5141. PMLR, 2022.

\bibitem[DKTZ24]{DKTZonline}
I.~Diakonikolas, V.~Kontonis, C.~Tzamos, and N.~Zarifis.
\newblock {Online Linear Classification with Massart Noise}, 2024.
\newblock Arxiv eprint: 2405.12958.

\bibitem[DZ24]{DiakonikolasZ24}
Ilias Diakonikolas and Nikos Zarifis.
\newblock A near-optimal algorithm for learning margin halfspaces with massart noise.
\newblock In {\em Advances in Neural Information Processing Systems 37: Annual Conference on Neural Information Processing Systems 2024}, 2024.

\bibitem[FGKP06]{FeldmanGKP06}
V.~Feldman, P.~Gopalan, S.~Khot, and A.~K. Ponnuswami.
\newblock New results for learning noisy parities and halfspaces.
\newblock In {\em 47\textsuperscript{th} Annual {IEEE} Symposium on Foundations of Computer Science {(FOCS)}}, pages 563--574. {IEEE} Computer Society, 2006.

\bibitem[GC23]{GaoC23}
Xing Gao and Yu~Cheng.
\newblock Robust matrix sensing in the semi-random model.
\newblock In {\em Advances in Neural Information Processing Systems 36: Annual Conference on Neural Information Processing Systems 2023}, 2023.

\bibitem[GGKS23]{gollakota2023agnostically}
A.~Gollakota, P.~Gopalan, A.~R. Klivans, and K.~Stavropoulos.
\newblock Agnostically learning single-index models using omnipredictors.
\newblock {\em arXiv preprint arXiv:2306.10615}, 2023.

\bibitem[GR06]{GuruswamiR06}
Venkatesan Guruswami and Prasad Raghavendra.
\newblock Hardness of learning halfspaces with noise.
\newblock In {\em 47th Annual {IEEE} Symposium on Foundations of Computer Science (FOCS 2006)}, pages 543--552. {IEEE} Computer Society, 2006.

\bibitem[Hau92]{Haussler92}
David Haussler.
\newblock Decision theoretic generalizations of the {PAC} model for neural net and other learning applications.
\newblock {\em Inf. Comput.}, 100(1):78--150, 1992.

\bibitem[HKLM20]{HKLM20}
M.~Hopkins, D.~M. Kane, S.~Lovett, and G.~Mahajan.
\newblock Noise-tolerant, reliable active classification with comparison queries.
\newblock In {\em COLT}, 2020.

\bibitem[JLM{\etalchar{+}}23]{JambulapatiLMSST23}
Arun Jambulapati, Jerry Li, Christopher Musco, Kirankumar Shiragur, Aaron Sidford, and Kevin Tian.
\newblock Structured semidefinite programming for recovering structured preconditioners.
\newblock In {\em Advances in Neural Information Processing Systems 36: Annual Conference on Neural Information Processing Systems 2023}, 2023.

\bibitem[Kea98]{Kearns:98}
M.~J. Kearns.
\newblock Efficient noise-tolerant learning from statistical queries.
\newblock {\em Journal of the ACM}, 45(6):983--1006, 1998.

\bibitem[KIT{\etalchar{+}}23]{KontonisITBMV23}
Vasilis Kontonis, Fotis Iliopoulos, Khoa Trinh, Cenk Baykal, Gaurav Menghani, and Erik Vee.
\newblock Slam: Student-label mixing for distillation with unlabeled examples.
\newblock In {\em Advances in Neural Information Processing Systems 36: Annual Conference on Neural Information Processing Systems 2023}, 2023.

\bibitem[KKMS05]{KKMS:05}
A.~Kalai, A.~Klivans, Y.~Mansour, and R.~Servedio.
\newblock Agnostically learning halfspaces.
\newblock In {\em Proceedings of the 46\textsuperscript{th} IEEE Symposium on Foundations of Computer Science (FOCS)}, pages 11--20, 2005.

\bibitem[KLL{\etalchar{+}}23]{KelnerLLST23}
Jonathan~A. Kelner, Jerry Li, Allen Liu, Aaron Sidford, and Kevin Tian.
\newblock Semi-random sparse recovery in nearly-linear time.
\newblock In {\em The Thirty Sixth Annual Conference on Learning Theory, {COLT} 2023}, volume 195 of {\em Proceedings of Machine Learning Research}, pages 2352--2398. {PMLR}, 2023.

\bibitem[KOS08]{KOS:08}
A.~Klivans, R.~O'Donnell, and R.~Servedio.
\newblock Learning geometric concepts via {G}aussian surface area.
\newblock In {\em Proc.\ 49th IEEE Symposium on Foundations of Computer Science (FOCS)}, pages 541--550, Philadelphia, Pennsylvania, 2008.

\bibitem[KS09]{kalai2009isotron}
A.~T. Kalai and R.~Sastry.
\newblock The isotron algorithm: High-dimensional isotonic regression.
\newblock In {\em COLT}. Citeseer, 2009.

\bibitem[KSS94]{KearnsSS94}
Michael~J. Kearns, Robert~E. Schapire, and Linda Sellie.
\newblock Toward efficient agnostic learning.
\newblock {\em Mach. Learn.}, 17(2-3):115--141, 1994.

\bibitem[LS11]{LongServedio:11malicious}
P.~Long and R.~Servedio.
\newblock Learning large-margin halfspaces with more malicious noise.
\newblock {\em NIPS}, 2011.

\bibitem[MN06]{Massart2006}
P.~Massart and E.~Nedelec.
\newblock Risk bounds for statistical learning.
\newblock {\em Ann. Statist.}, 34(5):2326--2366, October 2006.

\bibitem[NT22]{NasserT22}
R.~Nasser and S.~Tiegel.
\newblock Optimal {SQ} lower bounds for learning halfspaces with massart noise.
\newblock In {\em Conference on Learning Theory}, volume 178 of {\em Proceedings of Machine Learning Research}, pages 1047--1074. {PMLR}, 2022.

\bibitem[Ros58]{Rosenblatt58}
F.~Rosenblatt.
\newblock The perceptron: a probabilistic model for information storage and organization in the brain.
\newblock {\em Psychological Review}, 65:386--407, 1958.

\bibitem[Slo88]{Sloan88}
R.~H. Sloan.
\newblock Types of noise in data for concept learning.
\newblock In {\em Proceedings of the First Annual Workshop on Computational Learning Theory}, COLT '88, pages 91--96, San Francisco, CA, USA, 1988. Morgan Kaufmann Publishers Inc.

\bibitem[SSS09]{SSS:09}
S.~Shalev Shwartz, O.~Shamir, and K.~Sridharan.
\newblock Agnostically learning halfspaces with margin errors.
\newblock TTI Technical Report, 2009.

\bibitem[Vai96]{Vaidya:96}
P.~M. Vaidya.
\newblock A new algorithm for minimizing convex functions over convex sets.
\newblock {\em Math.\ Prog.}, 73(3):291--341, 1996.

\bibitem[ZL21]{zhang2021improved}
Chicheng Zhang and Yinan Li.
\newblock Improved algorithms for efficient active learning halfspaces with massart and tsybakov noise.
\newblock In {\em Conference on Learning Theory}, pages 4526--4527. PMLR, 2021.

\bibitem[ZLC17]{ZhangLC17}
Y.~Zhang, P.~Liang, and M.~Charikar.
\newblock A hitting time analysis of stochastic gradient langevin dynamics.
\newblock In {\em Proceedings of the 30\textsuperscript{th} Conference on Learning Theory, {COLT} 2017}, pages 1980--2022, 2017.

\bibitem[ZWDD24]{ZWDD24}
Nikos Zarifis, Puqian Wang, Ilias Diakonikolas, and Jelena Diakonikolas.
\newblock Robustly learning single-index models via alignment sharpness.
\newblock {\em CoRR}, abs/2402.17756, 2024.

\end{thebibliography}

\newpage
\appendix

\section{Learning with unknown noise rate}
\label{sec:tolerant_noise}
In this section, we give an overview of how to obtain the same sample complexity as \Cref{theorem:halfspace-massart} even when the noise rate $\eta$ is unknown to the learner. 
The argument is standard and is as follows: we argue that our separating hyperplane (\Cref{lem:gradient_lb}) is tolerant to $O(\epsilon)$ noise in the parameter $\eta$. We can then discretize the interval $[0,\half]$ into intervals of size $\epsilon$ and run the training algorithm multiple times for these different choices. Then, we can output the hypothesis with lowest validation error among the classifiers output by these different runs of the algorithm. 

We now provide the argument that this idea indeed works, i.e., that \Cref{lem:gradient_lb} is robust.

\begin{lemma}
    \label{lem:tolerant_gradient_lb}
    If $D$ is an instance of the $\eta$-Massart halfspace model with margin $\gamma$, and $\w \in \R^d$ has $\lzo(\w)  \ge \eta + \eps$, then for any $\tilde{\beta}\in (1 - 2\eta -\epsilon,1 - 2\eta]$,
\[
\E_{(\x, y) \sim D}\left[ (\tilde{\beta} \sgn(\vec w \cdot \x) - y) \frac{\x}{|\vec w \cdot \x| + \gamma} \right] \cdot (\vec w - \vw^\star)  \geq 2 \eps.
\]
\end{lemma}
\begin{proof}
    The proof is almost identical to the proof of \Cref{lem:gradient_lb} except for a few steps. We highlight the differences. We reuse the notation from the previous proof.

    First consider $\x\notin A$, we observe that using the same argument as before, we now obtain
    \[
    g(\x)\geq \tilde{\beta}+\beta(\x)\geq  \beta+\beta(\x)-\epsilon.
    \]

    In the case of $\x\in A$, since $\beta(\x)\geq\beta\geq \tilde{\beta}$, we obtain 
    \[
    g(\x)\geq (\tilde{\beta}-\beta(\x))\geq \beta-\beta(\x)-\epsilon.
    \]

Using this, we can can complete the proof by repeating the steps of the previous proof.
\end{proof}

Having proved this, our algorithm is simple: run over the $\frac{1}{2\eps}$ choices of $\tilde{\beta}$ in $[0,\half]$ and run the algorithm from \Cref{theorem:halfspace-massart} for each choice, reusing the same samples in the different runs of the algorithm.
The correctness follows \Cref{lem:tolerant_gradient_lb} and the proof of \Cref{theorem:halfspace-massart} since one of the choices of parameters must lie in the interval $(1 - 2\eta -\epsilon, 1 - 2\eta]$.

\section{Potential improvements to \Cref{theorem:glm_intro}}\label{sec:glm_potential}

It is natural to ask whether the $\approx \eps^{-2}$ overhead of \Cref{theorem:glm_intro}, when compared to \Cref{thm:halfspace_intro}, can be removed. This overhead comes from the use of a weaker padding procedure in \Cref{lem:separating-hyperplane-massart-glm}, which can only guarantee a norm bound of $\approx (\eps\gamma)^{-1}$ rather than $\approx \gamma^{-1}$ in our gradient estimate. 

In this section, we sketch a strategy for reducing the sample complexity of learning Massart GLMs to ${O}((\eps\gamma)^{-2}))$ and describe the reasons why na\"ive implementations of this strategy fail.

First, let us describe why the rescaling used in the halfspace case, i.e., $(|\vw\cdot \x|+\gamma)^{-1}$ fails for GLMs. The difficulty is in the ``$B_2$'' case of \Cref{lem:separating-hyperplane-massart-glm}, where we use a multiple of $\beta(\x)$ to approximately pay for both $\beta(\x)$ and $|\sigma(\ws \cdot \x)|$.
Observe that in \Cref{eqn:advantage_glm}, the lower bound when $\alpha=1$ is exactly $\beta(\x)$, whereas the lower bound of $(2-\epsilon)\beta(\x)\geq \beta(\x)+|\sigma(\vw^\star\cdot \x)|-\epsilon$ achieved with a smaller $\alpha$ is crucial for our analysis, as described earlier. To avoid this, we propose a different approach. 

Our proposed approach is based on enforcing an artificial margin with respect to the current iterate $\vw$. Specifically, note that if $|\vw\cdot \x|$ was $\Omega(\gamma)$, then there is no need to add $\gamma$ in the denominator of the re-weighting. Instead, we can directly use \Cref{lem:glm_gradient_sketch} and $|\vw\cdot \x|^{-1} =  O(\gamma^{-1})$. Here we describe a way of enforcing this margin. We define the ``push-away'' operator $\T:\B^{d}\to \B^{d}$ as follows:
    \begin{equation}\label{eq:pushaway}
    \T(\x)=
    \begin{cases}
      \x &  |\vw\cdot \x|\geq \frac \gamma 3 \\
      \frac{\x+\frac{\gamma}{3}\cdot \sign(\vw\cdot \x)\widehat{\vw}}{\norm{\x+\frac{\gamma}{3}\cdot \sign(\vw\cdot \x)\widehat{\vw}}} &\text{ otherwise}
    \end{cases},
    \end{equation}
    Here and throughout, $\widehat{\vw} \defeq \frac{\vw}{\norm{\vw}}$ is the projection of $\vw$, clear from context, to the surface of $\ball^d$.
    
    It is straightforward to show by casework that for all $\x\in \B^{d}$, we have $|\vw\cdot \T(\x)|\geq \Omega(\gamma)\cdot \norm{\vw}$.  Also, for $\x \in \supp(D_{\x})$, $\sign(\vw^\star\cdot \x)=\sign(\vw^\star\cdot \T(\x))$ as $D_{\x}$ has $\Omega(\gamma)$ margin with respect to $\vw^\star$. For completeness, we provide a short proof of these properties here.
    
    \begin{lemma}[Properties of the push-away operator]
\label{lem:push_away_prop}
Given a vector $\vw\in \R^d$ and $\gamma\in (0,1]$, the function $\T$ defined in \eqref{eq:pushaway} has the following properties.
\begin{enumerate}
    \item For any $\x\in \B^d$, it holds that $|\vw\cdot \T(\x)|\geq \frac{\gamma}{6}\norm{\vw}$ and $\sign\big(\vw\cdot \T(\x)\big)=\sign(\vw\cdot \x)\,.$
    \item For any $\vv,\x\in \B^d$ with $|\vv\cdot \x|\geq \gamma$, it holds that $\sign(\vv\cdot \x)=\sign\big(\vv\cdot \T(\x)\big)$ and $|\vv\cdot \T(\x)|\geq \frac \gamma 3$. %
\end{enumerate}
\end{lemma}
\begin{proof}
The first conclusion is obvious if $|\vw \cdot \x| \ge \frac \gamma 3$, and otherwise follows from the facts that $\big|\vw\cdot\big(\x+\frac{\gamma}{3}\sign(\vw\cdot \x)\widehat{\vw}\big)\big|=\big|\vw\cdot \x+\frac{\gamma}{3}\sign(\vw\cdot \x)\norm{\vw}\big|\geq \frac{\gamma}{3}\norm{\vw}$ and $\norm{\x+\frac{\gamma}{3}\sign(\vw\cdot \x)\tilde{\vw}}\leq 2$. The fact that $\sign\big(\vw\cdot \T(\x)\big)=\sign(\vw\cdot \x)$ follows directly from the definition of the operator.

We now prove the second conclusion. Since $\norm{\x+\frac{\gamma}{3}\sign(\vw\cdot \x)\widehat{\vw}}\leq 2$, we have that 
\[|\vv\cdot \T(\x)|\geq \half \big|\vv\cdot \x+\frac{\gamma}{3}\sign(\vw\cdot \x)\widehat{\vw}\cdot \vv\big|\geq \half \Par{\gamma - \frac \gamma 3} = \frac \gamma 3.\]
Similarly, $\sign(\vv\cdot \x)=\sign\big(\vv\cdot \T(\x)\big)$ since $\sign(\vv\cdot \x + a) =\sign(\vv\cdot \x)$ for $|a| \le \frac \gamma 3$.
\end{proof}
    To apply our observation, before updating the vector $\vw$, we may transform the input sample $\x\mapsto \T(\x)$ and then treat $\T(\x)$ as the sample. Since the sign of the true halfspace is unchanged and $\T(\x)$ now has $\Omega(\gamma)\cdot \norm{\vw}$ margin with respect to the halfspace $\vw$, \Cref{lem:glm_gradient_sketch} implies that there is a separating oracle $\vec g$ such that $\vec g\cdot (\vw-\vw^\star)\geq \epsilon$ and $\norm{\vec g} = O(\gamma^{-1} \norm{\vw}^{-1})$. 

    Thus, as long as we can enforce $\norm{\vw}\geq \Omega(1)$ for all iterates $\vw$ in our iterative method, using $\vec g$ to update the halfspace $\vw$ should imply a sample complexity of $O((\eps\gamma)^{-2})$. Because our hypotheses are always halfspaces, which are invariant to the scale of $\vw$, it is reasonable to conjecture there is a way to enforce this constraint. However, the non-convexity of the surface of the unit ball $\ball^d$ as a constraint set causes difficulties when combining with the iterative method regret analysis.
    
    Here we describe two natural attempts to enforce $\norm{\vw} = \Omega(1)$ throughout the course of our iterative method, and the corresponding bottlenecks encountered in the analysis. We leave overcoming these bottlenecks as an interesting open direction for future work.
    
    \paragraph{Non-convex projections.} One straightforward approach to enforcing a lower bound on $\norm{\vw}$ is by modifying our projected gradient updates to only be over the surface of the unit $\ball^d$. An equivalent view is: rather than using a scaled gradient estimate $\lam \vec g$ to update $\vw$, we can instead update
    $\vw\leftarrow\vw-\lambda \vec g^{\perp}$ (before projection), where $\vec g^{\perp}$ is the component of $\vec g$ perpendicular to $\vw$. This fixes the issue of lower bounding $\norm{\vw}$; it remains to show $\vec g^{\perp}\cdot (\vw-\vw^\star)$ is large enough. 
    
    With a modified gradient $\vec g\defeq (\sigma(\vw\cdot \x\cos \theta)-y)\frac{\T(\x)}{|\vw\cdot \T(\x)|}$, where $\theta$ is the angle between $\vw$ and $\vw^\star$, we were able to show that $\g^{\perp}\cdot (\vw-\vw^\star)\geq \epsilon \cos \theta$. This suggests that we can make geometric improvement on $\theta$. However, we were unable to complete this argument due to two hurdles. 
    The first issue is caused by the stochastic nature of our update, and the second is simply that we do not know $\theta$, and hence must proceed via using estimates of it which hold with small failure probability. 
    
    With a randomly-initialized $\vw^0$, we can ensure $\cos \theta^{0}= \Omega(d^{-1/2})$ with constant probability. However, it is not clear how to argue $\cos \theta$ does not drop after this. If the update were deterministic instead of stochastic, then we can argue that the angle reduces deterministically and $\cos \theta$ roughly doubles every step resulting in a sample complexity that is at most $O(\log (d) \cdot (\gamma\eps)^{-2})$. However, since our updates are stochastic, it is unclear how to continue the argument as it appears the noise in our gradient estimates may dominate the deterministic signal making progress. Moreover, our argument that using our modified gradient $\vec g$ suffices to make progress is rather brittle to the estimate of $\theta$. In particular, we cannot enforce tight-enough bounds on $\theta$ due to the stochasticity of gradient estimates for the rest of the argument to go through.
    
    \paragraph{Norm-dependent step sizes.} Another approach to enforce a scale-dependent progress bound is by modifying our step size scaling to take iterate norms into account. This strategy is advantageous compared to the prior angle-based modification, in that $\norm{\vw}$ for an iterate $\vw$ is always a known quantity, so we do not need to estimate bounds on it. However, this idea is also hindered by our stochastic updates, which cause too much instability in $\norm{\vw}$ to make sufficient progress.

    Concretely, for $\vw$ satisfying $\lzo(\vw) \ge \optrcn + \frac \tau 2 + \Omega(\eps)$, using the rescaled gradient estimate
    \[\g^t\defeq(\sigma(\widehat{\vw}^t \cdot \x)-y) \cdot \frac{\Tt(\x)}{|\vw^t\cdot \Tt(\x)|}\cdot\norm{\vw^t} = (\sigma(\widehat{\vw}^t \cdot \x)-y) \cdot \frac{\Tt(\x)}{|\widehat{\vw}^t\cdot \Tt(\x)|},\] 
    we obtain that $\g^t\cdot (\widehat{\vw}^t-\vw^\star)\geq \epsilon$ using a similar argument to \Cref{lem:separating-hyperplane-massart-glm}. Also, note that because the push-away operator $\Tt$ enforces a margin, we have the bound $\|\vg^t\| = O(\gamma^{-1})$.
    
    It remains to show that using these updates in our iterative method efficiently yields a     
    low-error halfspace $\vw$. Following \Cref{lem:sgd}, a natural  potential to track is $\Phi_t \defeq \|\widehat{\vw}^t -\vw^\star\|^2$. Letting 
    \[1+\delta_t \defeq \frac{\norm{\vw^{t}}}{\norm{\vw^{t + 1}}}\] 
    govern the iteration's norm stability,
    by using norm-dependent step sizes $\lam_t \gets \lam \norm{\vw^{t}}$, we have
    \begin{align*}
        \Phi_{t + 1} &=\norm{\widehat{\vw}^{t + 1}-\vw^\star}^2\\
        &\le \left\|{\frac{\vw^{t}-\lambda_t\g^t}{\norm{\vw^{t + 1}}}-\vw^\star}\right\|^2 = \left\|\widehat{\vw}^{t}\left(\frac{\norm{\vw^{t}}}{\norm{\vw^{t + 1}}}\right)-\vw^\star-\frac{\lambda_t\g^t}{\norm{\vw^{t + 1}}}\right\|^2\\
        &= \left\|\widehat{\vw}^{t}(1+\delta_t)-\vw^\star-\frac{\lambda_t\g^t}{\norm{\vw^{t + 1}}}\right\|^2 \\
        &\leq \left(\left\|\widehat{\vw}^{t}-\vw^\star-\frac{\lambda_t\g^t}{\norm{\vw^{t + 1}}}\right\|+|\delta_t|\right)^2\\
        &= \Phi_{t} -2\frac{\lambda_t\g^t}{\norm{\vw^{t + 1}}}\cdot(\widehat{\vw}^{t}-\vw^\star)+\lambda^2_t\frac{\norm{\g^t}^2}{\norm{\vw^{t + 1}}^2}+2|\delta_t|\left\|\widehat{\vw}^{t}-\vw^\star-\frac{\lambda_t\g^t}{\norm{\vw^{t + 1}}}\right\|+\delta^2_t\\
        &\leq \Phi_{t} -2\lambda\epsilon+O\Par{\frac{\lam^2}{\gamma^2}} +O\Par{ |\delta_t|\Par{ \lam \eps + 1 + \frac{\lam}{\gamma}} + \delta_t^2}.
    \end{align*}
    If the right-hand side only contained the first three terms, i.e., $\delta_t = 0$, then choosing $\lambda=\Theta(\gamma^2\epsilon)$ makes $\Phi_t$ drop to $0$ in $O((\eps\gamma)^{-2})$ steps with constant probability (\Cref{lem:sgd}) as desired.
    
    However, it is not clear how to control the remaining terms involving $\delta_t$. This parameter controls the stability of iterate norms. It is straightforward to check that due to the form of our updates, i.e., $\vw_{t + 1} \gets \vw_t - \lam_t \vec g_t$ (before projection),
    $|\delta_t| = O(\lam + \frac{\lam^2}{\gamma^2})$:
    \begin{align*}\norm{\vw_{t + 1}}^2 &\le \norm{\vw_t}^2 -\underbrace{2 \lam_t (\sigma(\widehat{\vw}^t \cdot \x)-y) \cdot \frac{\vw^t \cdot \Tt(\x)}{|\widehat{\vw}^t\cdot \Tt(\x)|}}_{\in [\pm O(\lam) \norm{\vw^t}^2]} + \underbrace{\lam_t^2 \norm{(\sigma(\widehat{\vw}^t \cdot \x)-y) \cdot \frac{\Tt(\x)}{|\widehat{\vw}^t\cdot \Tt(\x)|}}^2}_{\in [\pm O(\lam^2 \gamma^{-2}) \norm{\vw^t}^2]}\\
    &\le  \norm{\vw_t}^2 \cdot \Par{1 \pm O\Par{\lam + \frac{\lam^2}{\gamma^2}}}.
    \end{align*}
    Unfortunately, this bound is insufficient to guarantee convergence. For instance, the additional instability term $O(|\delta_t|)$ may already scale as $\Theta(\lam)$, outweighing the $\Omega(\lam\eps)$ progress achieved. Thus, it seems we need a better bound on iterate norm stability to go forward with this approach. 
\end{document}